\documentclass[10pt,twocolumn,letterpaper]{article}

\usepackage{cvpr}
\usepackage{epsfig}
\usepackage{times}
\usepackage{helvet}
\usepackage{courier}
\usepackage{graphicx} % more modern
\usepackage{subfigure}
\usepackage{algorithm}
\usepackage{algorithmic}

\usepackage{placeins}
\usepackage{color}
\usepackage{amsmath,amssymb,amsfonts,amsbsy,amsthm}
\newcommand{\bh}{\mbox{\boldmath $h$}}

\newcommand{\bm}{\mbox{\boldmath $m$}}

\newcommand{\bs}{\mbox{\boldmath $s$}}

\newcommand{\bw}{\mbox{\boldmath $w$}}
\newcommand{\bx}{\mbox{\boldmath $x$}}
\newcommand{\by}{\mbox{\boldmath $y$}}

\newcommand{\bA}{\mbox{\boldmath $A$}}

\newcommand{\bI}{\mbox{\boldmath $I$}}

\newcommand{\bS}{\mbox{\boldmath $S$}}

\newcommand{\bU}{\mbox{\boldmath $U$}}

\newcommand{\bX}{\mbox{\boldmath $X$}}

\newcommand{\calD}{{\mathcal{D}}}

\newcommand{\calF}{{\mathcal{F}}}

\newcommand{\calL}{{\mathcal{L}}}

\newcommand{\calN}{{\mathcal{N}}}
\newcommand{\calO}{{\mathcal{O}}}

\newcommand{\Real}{\mathbb R}

\newcommand{\bzero}{\mbox{\boldmath $0$}}

\newcommand{\be}{\begin{eqnarray}}
\newcommand{\ee}{\end{eqnarray}}
\newcommand{\bee}{\begin{eqnarray*}}
\newcommand{\eee}{\end{eqnarray*}}

\newcommand{\matrixb}{\left[ \begin{array}}
\newcommand{\matrixe}{\end{array} \right]}   

\newtheorem{theorem}{{Theorem}}

\def\abovestrut#1{\rule[0in]{0in}{#1}\ignorespaces}
\def\belowstrut#1{\rule[-#1]{0in}{#1}\ignorespaces}

\def\abovespace{\abovestrut{0.15in}}

\def\belowspace{\belowstrut{0.10in}}

\usepackage[pagebackref=true,breaklinks=true,letterpaper=true,colorlinks,bookmarks=false]{hyperref}

\cvprfinalcopy % *** Uncomment this line for the final submission

\def\cvprPaperID{1707} % *** Enter the CVPR Paper ID here

\ifcvprfinal\fi
\begin{document}

%%%%%%%%% TITLE
\title{Learning to Select Pre-trained Deep Representations with \\
Bayesian Evidence Framework}

\author{Yong-Deok Kim\thanks{This work
was done when Y. Kim and T. Jang were with POSTECH.} $^1$ ~~~ Taewoong Jang$^{*2}$ ~~~ Bohyung Han$^3$ ~~~ Seungjin Choi$^3$\\
$^1$Software R\&D Center, Device Solutions, Samsung Electronics, Korea\\
$^2$Stradvision Inc., Korea \\
$^3$Department of Computer Science and Engineering, POSTECH, Korea \\
{\tt\small yd.mlg.kim@samsung.com} ~
{\tt\small taewoong.jang@stradvision.com} ~
{\tt\small \{bhhan,seungjin\}@postech.ac.kr}
}

\maketitle
\thispagestyle{empty}

%%%%%%%%% ABSTRACT
\begin{abstract}
We propose a Bayesian evidence framework to facilitate transfer learning from pre-trained deep convolutional neural networks (CNNs).
Our framework is formulated on top of a least squares SVM (LS-SVM) classifier, which is simple and fast in both training and testing, and achieves competitive performance in practice.
The regularization parameters in LS-SVM is estimated automatically without grid search and cross-validation by maximizing evidence, which is a useful measure to select the best performing CNN out of multiple candidates for transfer learning; the evidence is optimized efficiently by employing Aitken's delta-squared process, which accelerates convergence of fixed point update.
The proposed Bayesian evidence framework also provides a good solution to identify the best ensemble of heterogeneous CNNs through a greedy algorithm. 
%We claim that such a simple and efficient formulation can achieve competitive performance to existing techniques.
Our Bayesian evidence framework for transfer learning is tested on 12 visual recognition datasets and illustrates the state-of-the-art performance consistently in terms of prediction accuracy and modeling efficiency.
\end{abstract}

%%%%%%%%% BODY TEXT
\section{Introduction}
\label{sec:introduction}

%Deep convolutional neural network (CNN) has drawn great attention in machine learning and computer vision communities because a lot of outstanding results have been reported recently in visual recognition problems~\cite{KrizhevskyA2012nips,SermanetP2014iclr}, several deep CNN models trained on large scale image repositories such as ImageNet~\cite{DengJ2009cvpr} are open to public in Berkeley's \emph{Caffe Model Zoo\footnote{https://github.com/BVLC/caffe/wiki/Model-Zoo}}, and high-performance computing systems such as GPUs and distributed clusters have become available.
%Deep convolutional neural network (CNN) has drawn great attention in machine learning and computer vision communities because a lot of outstanding results have been reported recently in visual recognition problems~\cite{KrizhevskyA2012nips,SermanetP2014iclr}, several deep CNN models trained on large scale image repositories such as ImageNet~\cite{DengJ2009cvpr} are open to public, and high-performance computing systems such as GPUs and distributed clusters have become available.

Image representations from deep CNN models trained for specific image classification tasks  turn out to be powerful even for general purposes~\cite{AzizpourH2014arxiv,Chatfield2014bmvc,DonahueJ2014icml,OquabM2014cvpr,RazavianAS2014cvpr} and useful for transfer learning or domain adaptation.
Therefore, CNNs trained on specific problems or datasets are often fine-tuned to facilitate training for new tasks or domains~\cite{AzizpourH2014arxiv,Chatfield2014bmvc,girshick2014cvpr,Nam16,noh2015learning,ZhangN2014eccv}, and an even simpler approach---application of off-the-shelf classification algorithms such as SVM to the representations from deep CNNs~\cite{DonahueJ2014icml}---is getting more attractive in many computer vision problems.
However, fine-tuning of an entire deep network still requires a lot of efforts and resources, and SVM-based methods also involve time consuming grid search and cross validation to identify good regularization parameters.
In addition, when multiple pre-trained deep CNN models are available, it is unclear which pre-trained models are appropriate for target tasks and which classifiers would maximize accuracy and efficiency.
Unfortunately, most existing techniques for transfer learning or domain adaptation are limited to empirical analysis or ad-hoc application specific approaches.

\begin{figure}[t]
\centering
\includegraphics[width=1\linewidth]{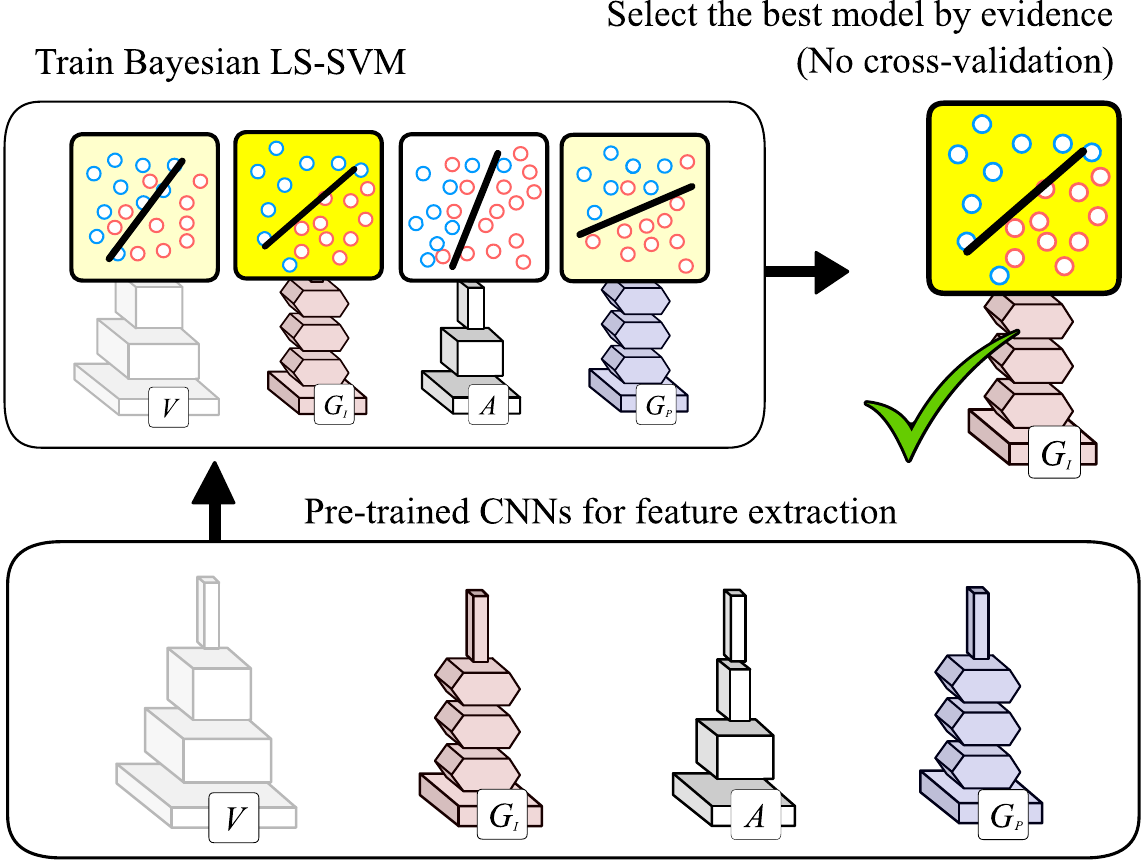}
\caption{We address a problem to select the best CNN out of multiple candidates as shown in this figure.
Additionally, our algorithm is capable of identifying the best ensemble of multiple CNNs to further improve performance.}
\label{fig:overview}
\end{figure}
We propose a simple but effective algorithm for transfer learning from pre-trained deep CNNs based on Bayesian least squares SVM (LS-SVM), which is formulated with Bayesian evidence framework~\cite{MacKayDJC92nc,GestelTV2002nc} and LS-SVM~\cite{SuykensJAK99npl}.
%\footnote{This is also known as a regularized least squares classifier~\cite{RifkinR2003nato}, which is similar to SVM but involves a least square loss.}.
This approach automatically determines regularization parameters in a principled way, and shows comparable performance to the standard SVMs based on hinge loss or squared hinge loss.
More importantly, Bayesian LS-SVM provides an effective solution to select the best CNN out of multiple candidates and identify a good ensemble of heterogeneous CNNs for performance improvement.
Figure~\ref{fig:overview} illustrates our approach.
We also propose a fast Bayesian LS-SVM, which maximizes the evidence more efficiently based on Aitken's delta-squared process~\cite{AitkenAC1927PRSE}.

One may argue against the use of LS-SVM for classification because the least squares loss function in LS-SVM tends to penalize well-classified examples.
However, least squares loss is often used for training multilayer perceptron~\cite{BishopCM95book} and shows comparable performance to SVMs~\cite{GestelTV2004mlj,ZhangP2004icpr}.
In addition, Bayesian LS-SVM provides a technically sound formulation with outstanding performance in terms of speed and accuracy for transfer learning with deep representations.
We also propose a fast Bayesian LS-SVM, which maximizes the evidence more efficiently based on Aitken�s delta-squared process~\cite{AitkenAC1927PRSE}. 
Considering simplicity and accuracy, we claim that our fast Bayesian LS-SVM is a reasonable choice for transfer learning with deep learning representation in visual recognition problems.
Based on this approach, we achieved promising results compared to the state-of-the-art techniques on 12 visual recognition tasks.
%Note that our formulation is generic for transfer learning but we focus on deep CNN models in this paper due to the limited amount of study in spite of the popularity of CNNs.

The rest of this paper is organized as follows.
Section~\ref{sec:related} describes examples of transfer learning or domain adaptation based on pre-trained CNNs for visual recognition problems.
Then, we discuss Bayesian evidence framework applicable to the same problem in Section~\ref{sec:bayesian} and its acceleration technique using Aitken's delta-squared process in Section~\ref{sec:fast}.
%Then, we discuss Bayesian evidence framework for transfer learning or domain adaptation based on pre-trained CNNs in Section~\ref{sec:bayesian}} and its acceleration technique using Aitken's delta-squared process in Section~\ref{sec:fast}.
The performance of our algorithm in various applications is demonstrated in Section~\ref{sec:experiments}.

\section{Related Work}
\label{sec:related}

%Although there exist Bayesian counterparts of the standard SVMs such as relevant vector machine~\cite{TippingME2001jmlr} or Gaussian process classification~\cite{RasmussenCE2006book}, these approaches are slow because they involve a matrix inversion in each iteration.

Since \emph{AlexNet}~\cite{KrizhevskyA2012nips} demonstrated impressive performance in the ImageNet large scale visual recognition challenge (LSVRC) 2012, a few deep CNNs with different architectures, \eg, {\em VGG}~\cite{SimonyanK2014arxiv} and {\em GoogLeNet}~\cite{SzegedyC2014arxiv}, have been proposed in the subsequent events.
%All of these three networks were trained independently although they used the same dataset.
%
Instead of training deep CNNs from scratch, some people have attempted to refine pre-trained networks for new tasks or datasets by updating the weights of all neurons or have adopted the intermediate outputs of existing deep networks as generic visual feature descriptors.
These strategies can be interpreted as transfer learning or domain adaptation.

Refining a pre-trained CNN is called fine-tuning, where the architecture of the network may be preserved while weights are updated based on new training data.
Fine-tuning is generally useful to improve performance~\cite{AzizpourH2014arxiv,Chatfield2014bmvc,girshick2014cvpr,ZhangN2014eccv} but requires careful implementation to avoid overfitting.
The second approach regards the pre-trained CNNs as feature extraction machines and combines the deep representations with the off-the-shelf classifiers such as linear SVM~\cite{DonahueJ2014icml,ZeilerMD2014eccv}, logistic regression~\cite{DonahueJ2014icml,ZeilerMD2014eccv}, and multi-layer neural network~\cite{OquabM2014cvpr}.
%The outputs from the first or second fully connected layers are typically used as feature descriptors.
The techniques in this category have been successful in many visual recognition tasks~\cite{AzizpourH2014arxiv,RazavianAS2014cvpr,SermanetP2014iclr}.
%where the third-last activations of \emph{OverFeat}~\cite{SermanetP2014iclr} are employed as generic image representations.

When combining a classification algorithm with image representations from pre-trained deep CNNs, we often face a critical issue.
Although several deep CNN models trained on large scale image repositories are publicly available, there is no principled way to select a CNN out of multiple candidates and find the best ensemble of multiple CNNs for performance optimization.
Existing algorithms typically rely on ad-hoc methods for model selection and fail to provide clear evidence for superior performance~\cite{AzizpourH2014arxiv}.

{\color{red}
}

%%%%%%%%%%%%%%%%%%%% Section %%%%%%%%%%%%%%%%%%%%%%%
\section{Bayesian LS-SVM for Model Selection}
\label{sec:bayesian}
This section discusses a Bayesian evidence framework to select the best CNN model(s) in the presence of transferable multiple candidates and identify a reasonable regularization parameter for LS-SVM classifier automatically.

\subsection{Problem Definition and Formulation}
\label{sub:problem}
Suppose that we have a set of pre-trained deep CNN models denoted by $\{\mathsf{CNN}_m | m=1 \dots M \}$.
Our goal is to identify the best performing deep CNN model among the $M$ networks for transfer learning.
A na{\"i}ve approach is to perform fine tuning of network for target task, which requires substantial efforts for training.
Another option is to replace some of fully connected layers in a CNN with an off-the-shelf classifier such as SVM and check the performance of target task through parameter tuning for each network, which would also be computationally expensive.

We adopt a Bayesian evidence framework based on LS-SVM to achieve the goal in a principled way, where the evidence of each network is maximized iteratively and the maximum evidences are used to select a reasonable model.
During the evidence maximization procedure, the regularization parameter of LS-SVM is identified automatically without time consuming grid search and cross-validation.
In addition, the Bayesian evidence framework is also applied to the construction of an ensemble of multiple CNNs to accomplish further performance improvement.

\subsection{LS-SVM}
\label{sub:ls}
%Before describing Bayesian evidence framework, we review the formulation of LS-SVM.
We deal with multi-label or multi-class classification problem, where the number of categories is $K$.
Let $\calD = \{ (\bx_n, y_n^{(k)}), k=1 \ldots K\}_{n = 1 \ldots N}$ be a training set, where $\bx_n \in \Real^D$ is a feature vector and $y_n^{(k)}$ is a binary variable that is set to 1 if label $k$ is given to $\bx_n$ and 0 otherwise.
Then, for each class $k$, we minimize a least squares loss with $L_2$ regularization penalty as follows:
\be
\label{eq:ls}
\min_{\bw^{(k)} \in \Real^D} \| \by^{(k)} - \bX^{\top} \bw^{(k)} \|^2 + \lambda^{(k)} \| \bw^{(k)} \|^2,
\ee
where $\bX = [\bx_1, \ldots, \bx_n] \in \Real^{D \times N}$ and $\by^{(k)} = [ y_1^{(k)}, \ldots, y_N^{(k)}]^{\top} \in \Real^{N}$.
The optimal solution of the problem in \eqref{eq:ls} is given by
\be
\label{eq:ls_solution}
\bw^{(k)} &=& (\bX\bX^{\top} + \lambda^{(k)} \bI)^{-1}\bX\by^{(k)}, \nonumber \\ 
&=& \bU(\bS + \lambda^{(k)} \bI)^{-1} \bU^{\top} \bX \by^{(k)},
\ee
where $\bU \bS \bU^{\top}$ is the eigen-decomposition of $\bX\bX^{\top}$ 
and $\bI$ is an identity matrix.
This regularized least squares approach has clear benefit that it requires only one eigen-decomposition of $\bX\bX^{\top}$ to obtain the solution in \eqref{eq:ls_solution} for all combinations of $\lambda^{(k)}$ and $\by^{(k)}$.

\subsection{Bayesian Evidence Framework}
\label{sub:bayesian}
The optimization of the regularized least squares formulation presented in \eqref{eq:ls} is equivalent to the maximization of the posterior with fixed hyperparamters $\alpha$ and $\beta$ denoted by $p(\bw |\by,\bX,\alpha,\beta)$, where $\lambda = \alpha / \beta$.
The posterior can be decomposed into two terms by Bayesian theorem as
\begin{align}
p(\bw |\by,\bX,\alpha,\beta) \propto p(\by | \bX, \bw, \beta) p(\bw | \alpha),
\label{eq:posterior}
\end{align}
where $p(\by | \bX, \bw, \beta)$ corresponds to Gaussian observation noise model given by
\begin{align}
\label{eq:likelihood}
p(\by | \bX, \bw, \beta) = \prod_{n=1}^N \calN (y_n | \bx_n^{\top} \bw, \beta^{-1} )
\end{align}
and
$p(\bw | \alpha)$ denotes a zero-mean isotropic Gaussian prior as
\begin{align}
\label{eq:prior}
p(\bw | \alpha) = \calN(\bw | \bzero, \alpha^{-1} \bI).
\end{align}
Note that we dropped superscript $(k)$ for notational simplicity from the equations in this subsection.

%Now we briefly review the Bayesian evidence framework~\cite{MacKayDJC92nc,GestelTV2002nc} for optimal regularization parameter selection.
In the Bayesian evidence framework~\cite{MacKayDJC92nc,GestelTV2002nc}, the evidence, also known as marginal likelihood, is a function of hyperparameters $\alpha$ and $\beta$ as
\be
p(\by | \bX, \alpha, \beta) = \int p(\by | \bX, \bw, \beta) p(\bw | \alpha) d\bw.
\label{eq:marginal}
\ee
%where model parameter $\bw$ is integrated out.
Under the probabilistic model assumptions corresponding to \eqref{eq:likelihood} and \eqref{eq:prior},
the log evidence $\calL(\alpha,\beta)$ is given by
\begin{align}
\calL(\alpha,\beta) \equiv~& \lefteqn{\log p(\by | \bX, \alpha, \beta)} \\
=~& \frac{D}{2} \log \alpha + \frac{N}{2} \log \beta - \frac{1}{2}\log |\bA| \nonumber \\
&- \frac{\beta}{2} \| \by - \bX^{\top} \bm \|^2
- \frac{\alpha}{2} \bm^{\top} \bm - \frac{N}{2} \log 2\pi, \nonumber
\end{align}
where the precision matrix and mean vector of the posterior $p(\bw | \by, \bX, \alpha, \beta) = \calN(\bw | \bm, \bA^{-1})$ are given respectively by
\be
\bA = \alpha \bI + \beta \bX\bX^{\top} ~~\text{and}~~
\bm = \beta \bA^{-1} \bX \by. \nonumber
\ee

The log evidence $\calL(\alpha,\beta)$ is maximized by repeatedly alternating the following fixed point update rules
\be
%\label{eq:alpha_fixed}
%\alpha & = & \frac{\gamma}{\bm^{\top}\bm} ~~\text{and} \\
%\label{eq:beta_fixed}
%\beta & = & \frac{N - \gamma}{\| \by - \bX^{\top}\bm \|^2},
\alpha = \frac{\gamma}{\bm^{\top}\bm} ~~~\text{and} ~~~ \beta = \frac{N - \gamma}{\| \by - \bX^{\top}\bm \|^2},
\label{eq:alpha_beta_fixed}
\ee
which involves the derivation of $\gamma$ as
\be
\label{eq:gamma}
\gamma = \sum_{d=1}^D \frac{ \beta s_d}{\alpha + \beta s_d } = \sum_{d=1}^D \frac{ s_d}{\lambda + s_d },
\ee
where $\{s_d\}_{d=1}^D$ are eigenvalues of $\bX\bX^{\top}$.
Note that $\bm$ and $\gamma$ should be re-estimated after each update of $\alpha$ and $\beta$.
{

%Another pair of update rules of $\alpha$ and $\beta$ are derived by an expectation-maximization (EM) technique, but it is substantially slower than the fixed point update rules presented in \eqref{eq:alpha_beta_fixed}.
Another pair of update rules of $\alpha$ and $\beta$ are derived by an expectation-maximization (EM) technique as
\be
\label{eq:alpha_em}
\alpha & = & \frac{D}{\bm^{\top}\bm + \mbox{Tr}(\bA^{-1})} ~~\text{and} \\
\label{eq:beta_em}
\beta & = & \frac{N}{\| \by - \bX^{\top}\bm \|^2 + \mbox{Tr}(\bA^{-1}\bX\bX^{\top})},
\ee
but these procedures are substantially slower than the fixed point update rules in \eqref{eq:alpha_beta_fixed}.
}

Through the optimization procedures described above, we determine the regularization parameter $\lambda = \alpha / \beta$.
Although the estimated parameters are not optimal, they may still be reasonable solutions since they are obtained by maximizing marginal likelihood in \eqref{eq:marginal}.

\subsection{Model Selection using Evidence}
\label{sub:model}
The evidence computed in the previous subsection is for a single class, and the overall evidence for entire classes, denoted by $\mathcal{L}^*$, is obtained by the summation of the evidences from individual classes, which is given by
\begin{align}
\mathcal{L}^* = \sum_{k=1}^K \mathcal{L}(\alpha^{(k)}, \beta^{(k)}).
\label{eq:overall_evidence}
\end{align}

We compute the overall evidence corresponding to each deep CNN model, and choose the model with the maximum evidence for transfer learning.
We expect that the selected model performs best among all candidates, which will be verified in our experiment.

In addition, when an ensemble of deep CNNs needs to be constructed for a target task, our approach selects a subset of good pre-trained CNNs in a greedy manner.
Specifically, we add a network with the largest evidence in each stage and test whether the augmented network improves the evidence or not.
The network is accepted if the evidence increases, or rejected otherwise.
After the last candidate is tested, we obtain the final network combination and its associated model learned with the concatenated feature descriptors from accepted networks.

%==========================================================
%==========================================================

\section{Fast Bayesian LS-SVM}
\label{sec:fast}
Bayesian evidence framework discussed in Section~\ref{sec:bayesian} is useful to identify a good CNN for transfer learning and a reasonable regularization parameter.
To make this framework even more practical, we present a faster algorithm to accomplish the same goal and a new theory that guarantees the converges of the algorithm.

\subsection{Reformulation of Evidence}
\label{sub:reformulation}
We are going to reduce $\calL(\alpha,\beta)$ to a function with only one parameter that directly corresponds to the regularization parameter $\lambda = \alpha / \beta$.
To this end, we re-write $\calL(\alpha,\beta)$ by using the eigen-decomposition $\bX\bX^{\top} = \bU \bS \bU^{\top}$ as
\begin{align}
\calL(\alpha,\beta) 
%=~& \frac{D}{2} \log \alpha + \frac{N}{2} \log \beta
%- \frac{1}{2} \sum_{d=1}^D \log (\alpha + \beta s_d) \nonumber \\
%& - \frac{\beta}{2} \by^{\top}\by \nonumber
%- \frac{1}{2} \bh^{\top} (\bL_1 + \bL_2 + \bL_3) \bh
%- \frac{N}{2} \log 2\pi \nonumber \\
&= \frac{D}{2} \log \alpha + \frac{N}{2} \log \beta
- \frac{1}{2} \sum_{d=1}^D \log (\alpha + \beta s_d) \nonumber \\
&  -\frac{\beta}{2} \by^{\top}\by
% + \frac{1}{2} \sum_{d=1}^D \left(\frac{\beta^2 }{\alpha + \beta s_d}	 \right) h_d^2 - \frac{N}{2} \log 2\pi,
 + \frac{\beta^2}{2} \sum_{d=1}^D \frac{h_d^2}{\alpha + \beta s_d} - \frac{N}{2} \log 2\pi,
\end{align}
where $s_d$ is the $d$-th diagonal element in $\bS$ and $h_d$ denotes the $d$-th element in $\bh = \bU^{\top}\bX\by$.
%\bee
%\bh &=& \bU^{\top}\bX\by, \\
%\bL_1 &=&  -2\beta^2 (\alpha \bI + \beta \bS)^{-1}, \\
%\bL_2 &=& \beta^3 (\alpha \bI + \beta \bS)^{-2}, \\
%\bL_3 &=&  \alpha \beta^2 (\alpha \bI + \beta \bS)^{-2}\bS.
%\eee
Then, we re-parameterize $\calL(\alpha,\beta)$ into $\calF(\lambda,\beta)$ as
\begin{align}
\label{eq:re-parameter}
{\calF(\lambda,\beta)} &= \frac{D}{2}\log \lambda + \frac{N}{2} \log \beta
- \frac{1}{2} \sum_{d=1}^D \log (\lambda + s_d) \nonumber \\
& - \frac{\beta}{2} \left(\by^{\top}\by -\sum_{d=1}^D \frac{h_d^2}{\lambda+s_d}\right ) - \frac{N}{2} \log 2 \pi.
\end{align}
{%By setting the derivative of $\calF(\lambda, \beta)$ to zero, we derive the following equation to obtain $\beta$:
The derivative of $\calF(\lambda, \beta)$ with respect to $\beta$ is given by
\bee
\frac{\partial \calF}{\partial \beta}
= \frac{N}{2\beta} - \frac{1}{2}\left(\by^{\top}\by -\sum_{d=1}^D \frac{h_d^2}{\lambda+s_d}\right ),
\eee
and we obtain the following equation by setting this derivative to zero,
}
\be
\label{eq:beta}
\beta = \frac{N}{\by^{\top}\by -\sum_{d=1}^D \frac{h_d^2}{\lambda+s_d}}.
\ee
Finally, we obtain a one-dimensional function of the log evidence by plugging \eqref{eq:beta} into \eqref{eq:re-parameter}, which is given by
\begin{align}
\label{eq:final_form}
\calF(\lambda) ~=&~ \frac{1}{2}\sum_{d=1}^D \log \frac{\lambda}{\lambda+s_d} + \frac{N}{2} \log N
- \frac{N}{2} -\frac{N}{2} \log 2\pi \nonumber \\
& - \frac{N}{2} \log\left(\by^{\top}\by -  \sum_{d=1}^D \frac{h_d^2}{\lambda + s_d}\right).
\end{align}
Figure~\ref{fig:log_evidence} illustrates the curvature of this log evidence function with respect to $\log \lambda$.
\begin{figure}[t]
\centering
\includegraphics[width=0.95\linewidth]{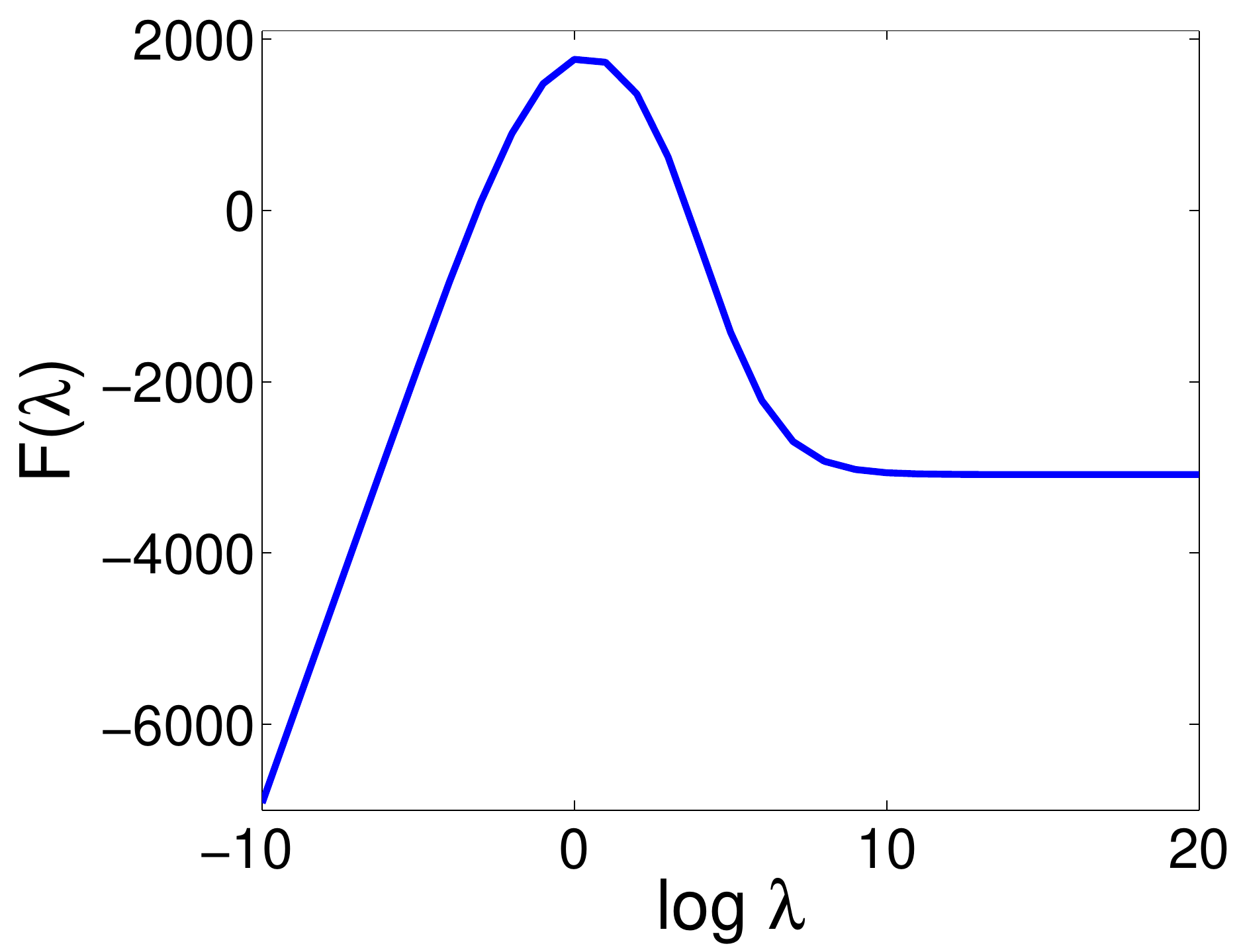}
\caption{Plot of the log evidence $\calF(\lambda)$ with respect to $\log \lambda$.
Note that $\calF(\lambda)$ is neither convex nor concave. 
}
\label{fig:log_evidence}
\end{figure}

\iffalse
\begin{figure*}[t]
\vskip 0.2in
\begin{center}
\centerline{\includegraphics[width=\columnwidth/2]{./fig_evidence.pdf} 
\includegraphics[width=\columnwidth/2]{./fig_aitken.pdf}
\includegraphics[width=\columnwidth/2]{./fig_aitken_false_case1.pdf}
\includegraphics[width=\columnwidth/2]{./fig_aitken_false_case2.pdf} }
\centerline{ \hspace{1.0cm} (a) \hspace{3.5cm} (b) \hspace{3.5cm} (c) \hspace{4cm} (d)}
\caption{(a) Plot of the log evidence $\calF(\lambda)$ versus $\log \lambda$.
(b) Illustration of Aitken's delta-squared process. 
(c), (d): False cases of Aitken's delta-squared process.}
\label{fig:method}
\end{center}
\vskip -0.2in
\end{figure*}
\fi

\subsection{New Fixed-point Update Rule}
We now derive a new fixed point update rule and present the sufficient condition for the existence of a fixed point.
The stationary points in \eqref{eq:final_form} with respect to $\lambda$ satisfy
\begin{align}
\frac{1}{2} \sum_{d=1}^D  \frac{s_d}{\lambda(\lambda + s_d)}
- \frac{N}{2} \frac{\sum_{d=1}^D \frac{h_d^2}{(\lambda+s_d)^2}}{\by^{\top}\by - \sum_{d=1}^D \frac{h_d^2}{\lambda+s_d}} = 0,
\end{align}
\iffalse
Multiply through by 2 and rearranging, we obtain
\bee
\frac{1}{\lambda}  \sum_{d=1}^D\frac{s_d}{\lambda + s_d}
= \left( \frac{N}{\by^{\top}\by - \sum_{d=1}^D \frac{h_d^2}{\lambda+s_d}} \right)
\left( \sum_{d=1}^D \frac{h_d^2}{(\lambda+s_d)^2}\right).
\eee
From this, we see that the value of $\lambda$ that maximizes the (\ref{eq:final_form}) satisfies
\be
\label{eq:update_lambda}
\lambda = \frac{\sum_{d=1}^D\frac{s_d}{\lambda + s_d}}
%{\left( \frac{N}{\by^{\top}\by - \sum_{d=1}^D \frac{h_d^2}{\lambda+s_d}} \right)
{\left( \frac{N}{\by^{\top}\by - \sum_{d=1}^D h_d^2 / (\lambda+s_d)} \right)
\left( \sum_{d=1}^D \frac{h_d^2}{(\lambda+s_d)^2}\right)} .
\ee
\fi
%
and we update the fixed-point by maximizing (\ref{eq:final_form}) as
\be
\label{eq:update_lambda}
\lambda = \frac{\sum_{d=1}^D\frac{s_d}{\lambda + s_d}}
%{\left( \frac{N}{\by^{\top}\by - \sum_{d=1}^D \frac{h_d^2}{\lambda+s_d}} \right)
{\left( \frac{N}{\by^{\top}\by - \sum_{d=1}^D h_d^2 / (\lambda+s_d)} \right)
\left( \sum_{d=1}^D \frac{h_d^2}{(\lambda+s_d)^2}\right)} .
\ee
%
\iffalse
The fixed point update rule in \eqref{eq:update_lambda} is re-written more concisely by using \eqref{eq:gamma} and \eqref{eq:beta} as
\bee
\lambda = \frac{\gamma}{\beta ~ \bm_N^{\top} \bm_N} = \frac{\gamma}{\beta ~ \sum_{d=1}^D \frac{h_d^2}{(\lambda + s_d)^2}}.
\eee
\fi
%
As illustrated in Figure~\ref{fig:log_evidence}, $\mathcal{F}(\lambda)$ in \eqref{eq:final_form} is neither convex nor concave as illustrated in the supplementary file.
However, we can show the sufficient condition of the existence of the fixed point using the following theorem. 
%See our supplementary file for the proof.
\begin{theorem}
Denote the update rule in \eqref{eq:update_lambda} by $f(\lambda)$.
%{\color{red}If $\|\by\|^2 \|\bX \|_F^2 < N \|\bX\by\|^2$,} then $f(\lambda)$ has a fixed point.
If $\by$ is a binary variable and $\bx_n$ is an $L_2$ normalized nonnegative vector, then $f(\lambda)$ has a fixed point.
\end{theorem}
\begin{proof}
We first show that $f(\lambda)$ is asymptotically linear in $\lambda$ as
\begin{align}
\lim_{\lambda \rightarrow \infty} \frac{f(\lambda)}{\lambda}
&= \lim_{\lambda \rightarrow \infty}
\frac{\left(\by^{\top}\by - \sum_{d=1} \frac{h_d^2}{\lambda+s_d} \right)
\sum_{d=1}^D \frac{s_d}{\lambda + s_d}}
{\lambda N \sum_{d=1}^D \frac{h_d^2}{(\lambda+s_d)^2}}  \nonumber \\
%&= \lim_{\lambda \rightarrow \infty}
%\frac{\left(\by^{\top}\by - \sum_{d=1} \frac{h_d^2}{\lambda+s_d} \right)
%\sum_{d=1}^D \frac{s_d}{1 + s_d / \lambda}}
%{ N \sum_{d=1}^D \frac{h_d^2}{1 + 2 s_d / \lambda + s_d / \lambda^2}} \nonumber \\
&= \frac{\by^{\top}\by \sum_{d=1}^D s_d}{ N \sum_{d=1}^D h_d^2}
= \frac{\|\by\|^2 \|\bX \|_F^2}{N \|\bX\by\|^2}. \nonumber
\end{align}
Since $\by$ is binary and $\bx_n$ is $L_2$ normalized and nonnegative, we can derive the following two relations,
\be
%\|\by\|^2 \|\bX \|_F^2 &=& \left( \sum_{n=1}^N y_n \right) N = P N ~~\text{and} \\
\hspace{-0.7cm} \|\by\|^2 \|\bX \|_F^2 \hspace{-0.2cm} &=& \hspace{-0.2cm} P N ~~\text{and} \label{eq:pn} \\
\hspace{-0.7cm} \| \bX \by \|^2 &=& \left( \sum_{n: y_n = 1} x_n \right)^2 > \sum_{n: y_n = 1} x_n^2 = P, \label{eq:p}
\ee
where $P = \sum_{n=1}^N y_n$.
From \eqref{eq:pn} and \eqref{eq:p}, it is shown that $ \|\by\|^2 \|\bX \|_F^2 < N \|\bX\by\|^2$.

Obviously, $f(0) > 0$ and there exists a $\lambda^{+}$ such that $f(\lambda^{+}) < \lambda^{+}$.
%If $ \|\by\|^2 \|\bX \|_F^2 < N \|\bX\by\|^2$, there exists a $\lambda^{+}$ such that $f(\lambda^{+}) < \lambda^{+}$.
The intermediate value theorem implies the existence of $\lambda^{\ast}$ such that $f(\lambda^{\ast}) = \lambda^{\ast}$, where $0 < \lambda^{\ast} < \lambda^+$ as illustrated in Figure~\ref{fig:aitken}.
\end{proof}

%
\iffalse
The sufficient condition $\|\by\|^2 \|\bX \|_F^2 < N \|\bX\by\|^2$ can be satisfied 
if $\by$ is binary variables (i.e. 0 or 1) and $\bx_n$ is $L2$ normalized nonnegative vectors
because
\bee
\|\by\|^2 \|\bX \|_F^2 &=& P N \\
\| \bX \by \|^2 &=& \| \sum_{n \in \Omega} x_n \|^2 \\
&=& \sum_{n \in \Omega} \| \bx_n \|^2 + 2 \sum_{ n \neq m} \bx_n^{\top}\bx_m > P,
\eee
where $P = \sum_{n=1}^N y_n$ and $\Omega$ is a set of indices of positive entries in $\by$.
Note that $\bx_n$ is trivially nonnegative in our problem setting because it is the activation of
rectified linear units of CNN.
\fi

%
\begin{figure}[t]
%\vskip 0.02in
\centering
\includegraphics[width=0.9\linewidth]{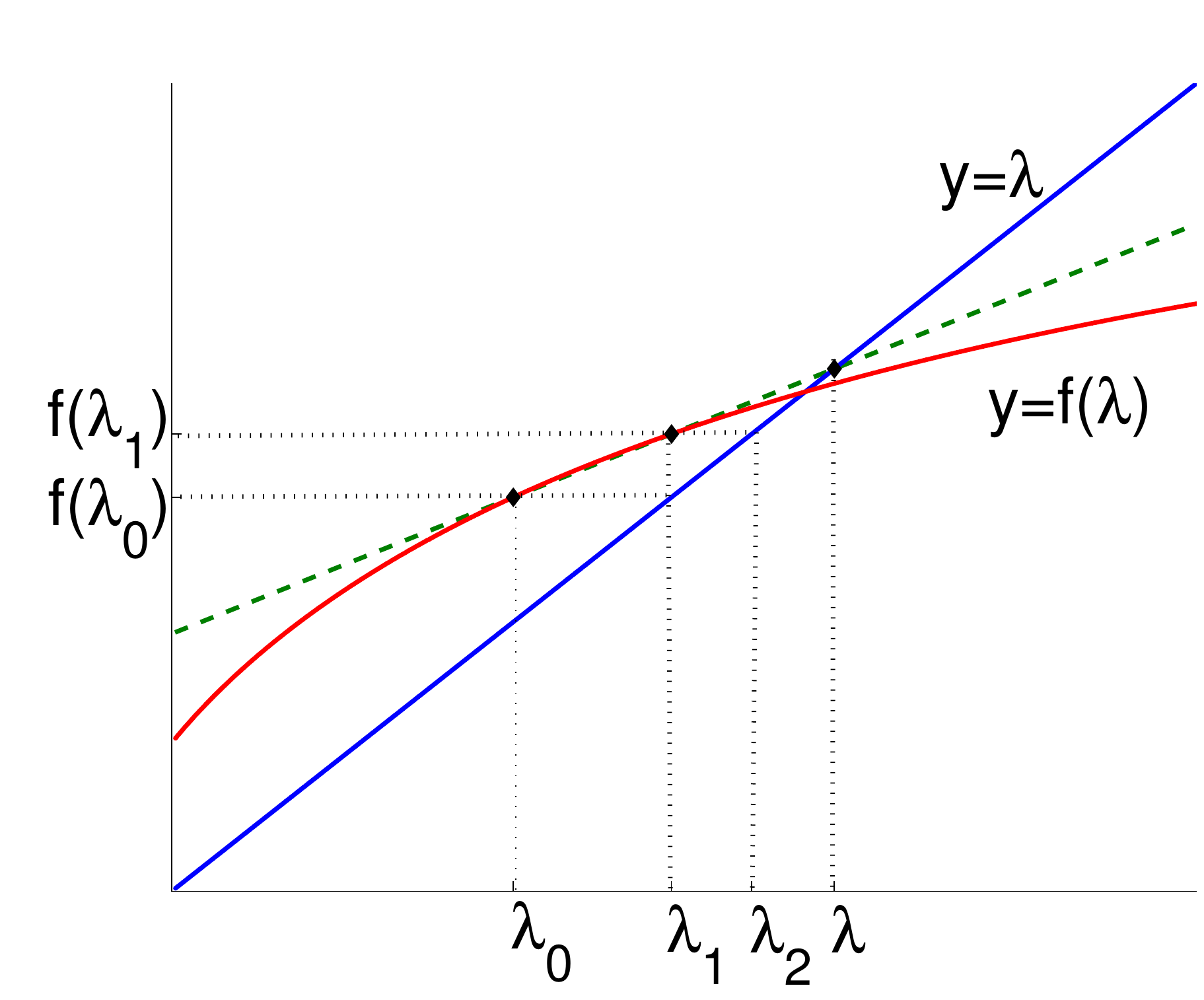}
\caption{Aitken's delta-squared process. The fixed point update function $f(\lambda)$
is approximated by green dashed line and its intersection with $y=\lambda$ becomes
the next update point.}
\label{fig:aitken}
\end{figure}
The fixed point is unique if $f(\lambda)$ is concave.
Although it is always concave according to our observation, we have no proof yet 
and leave it as a future work

%==========================================================

\subsection{Speed Up Algorithm}
\begin{algorithm}[t]
\begin{algorithmic}
   \caption{Fast Bayesian Least Squares}
   \label{alg:FBLS}
   \STATE {\bfseries Input: } $\bX \in \Real^{N \times D}$ and $\by \in \Real^N$.
   \STATE {\bfseries Output: } Optimal solutions $(\bw,\lambda)$.
   \STATE Initialize $\lambda$  ~~~~~~~~~~~~~~~~~// {\it e.g.}, $\lambda = 1$
   \STATE $(\bU,\bS) \leftarrow $ eigen-decomposition$(\bX\bX^{\top})$
   \STATE $\bs \leftarrow$ diag$(\bS), ~~\bh \leftarrow \bU^{\top}\bX\by$
   %\FOR{$t=1$ {\bfseries to} $T$ $\{$ e.g. $T = 0 \mbox{~or~} 30 \}$  }
   %\STATE $\lambda \leftarrow$ UPDATE$(\lambda,\bs,\bh,N,\by^{\top}\by)$
   %\ENDFOR
   \REPEAT
   \STATE $\lambda_0 \leftarrow \lambda$
   \STATE $\lambda_1 \leftarrow$ UPDATE $(\lambda_0,\bs,\bh,N,\by^{\top}\by)$
   \STATE $\lambda_2 \leftarrow$ UPDATE $(\lambda_1,\bs,\bh,N,\by^{\top}\by)$
   \STATE $\lambda \leftarrow \lambda_0 - \frac{(\lambda_1 - \lambda_0)^2}{(\lambda_2-\lambda_1)-		(\lambda_1 - \lambda_0)}$
   \IF{$\lambda < 0$ or $\lambda = \pm \infty$}
   \STATE $\lambda \leftarrow \lambda_2$
   \ENDIF
   \UNTIL { $|\lambda - \lambda_0| < \epsilon$ ~~~// {\it e.g.}, $\epsilon = 10^{-5}$   }
   \STATE $\bw \leftarrow \bU (\bS + \lambda \bI)^{-1} \bh$
\end{algorithmic}
\end{algorithm}
\begin{algorithm}[t]
\begin{algorithmic}
   \caption{$\lambda$ = UPDATE$(\lambda,\bs,\bh,N,\by^{\top}\by)$}
   \label{alg:update}
   \STATE $\gamma \leftarrow \sum_{d=1}^D \frac{s_d}{\lambda + s_d}$
   \STATE $\beta \leftarrow N / (\by^{\top}\by - \sum_{d=1}^D \frac{h_d^2}{\lambda + s_d})$
   \STATE $\bm^{\top} \bm \leftarrow \sum_{d=1}^D \frac{h_d^2}{(\lambda + s_d)^2}$
   \STATE $\lambda \leftarrow \frac{\gamma }{\beta ~ \bm^{\top}\bm}$
   \STATE {\bfseries return} $\lambda$
\end{algorithmic}
\end{algorithm}

We accelerate the fixed point update rule in \eqref{eq:update_lambda} by using Aitken's delta-squared process~\cite{AitkenAC1927PRSE}.
Figure~\ref{fig:aitken} illustrates the Aitken's delta-squared process.
Let's focus on the two points $(\lambda_0, f(\lambda_0))$ and $(\lambda_1, f(\lambda_1))$,
and line going through these two points.
The equation of this line is
\be
\label{eq:line}
y = \lambda_1 + (\lambda - \lambda_0) \frac{\lambda_2-\lambda_1}{\lambda_1 - \lambda_0},
\ee
where $f(\lambda_0)$ and $f(\lambda_1)$ are replaced by $\lambda_1$ and $\lambda_2$, respectively.
The idea behind Aitken's method is to approximate fixed point $\lambda^{\ast}$ using the intersection of the line in \eqref{eq:line} with line $y=\lambda$, which is given by
\be
\label{eq:aitken}
\lambda = \lambda_0 - \frac{(\lambda_1 - \lambda_0)^2}{(\lambda_2-\lambda_1)-(\lambda_1 - \lambda_0)}. 
\ee

Our fast Bayesian learning algorithm for the regularized least squares problem in \eqref{eq:ls}
is summarized in Algorithm~\ref{alg:FBLS}.
In our algorithm, we first compute the eigen-decomposition of $\bX\bX^{\top}$. 
This is the most time consuming part but needs to be performed only once since the result can be reused for every label in $\by$.
After that, we obtain the regularization parameter $\lambda$ through an iterative procedure.
\iffalse
The algorithm consists of two stages.
At the first stage, we compute the eigen-decomposition of $\bX\bX^{\top}$.
This is the most time consuming part, but we can reuse the result for every element in $\by$.
At the second stage, we optimize the regularization parameter $\lambda$.
\fi

When we apply the Aitken's delta-squared process, we have two potential failure cases as in Figure~\ref{fig:aitken-false}(a) and \ref{fig:aitken-false}(b).
The first case often arises if the initial $\lambda_0$ is far from the fixed point $\lambda^{\ast}$, and the second case occurs when the approximating line in \eqref{eq:line} is parallel to $y=\lambda$.
Fortunately, these failures rarely happen in practice and can be handled easily by skipping the procedure in \eqref{eq:aitken} and updating $\lambda$ with $\lambda_2$.
\begin{figure}[t]
%\vskip 0.02in
\centering
%\includegraphics[width=0.7\linewidth]{./fig_aitken_false_case1.pdf} \\
%\centerline{(a) $\lambda < 0$}
%\includegraphics[width=0.7\linewidth]{./fig_aitken_false_case2.pdf}
%\centerline{(b) $\lambda = \pm \infty$}
\includegraphics[width=0.495\linewidth]{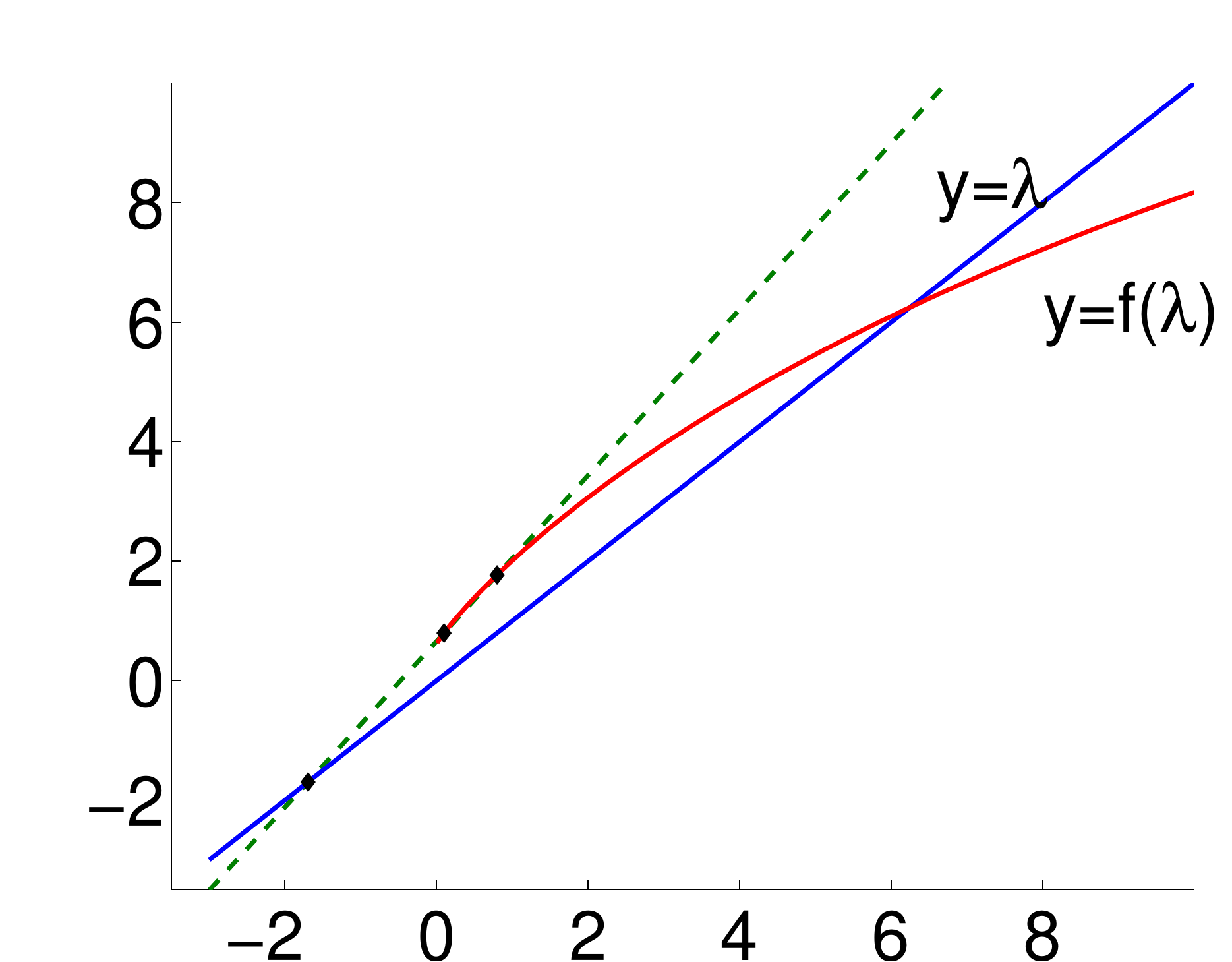}
\includegraphics[width=0.495\linewidth]{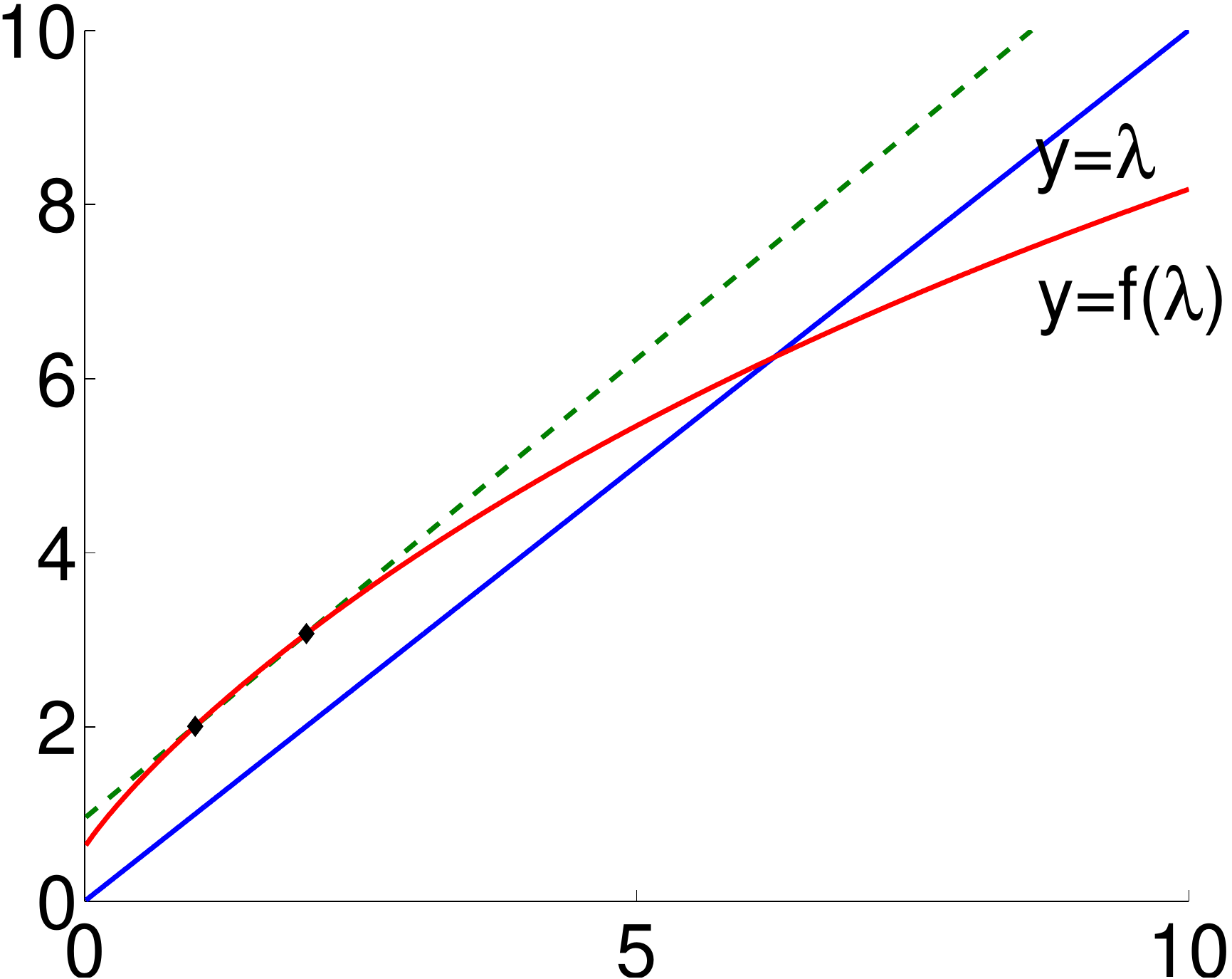}
%\centerline{\hspace{0.5cm} (a) $\lambda < 0$ \hspace{2.5cm} (b) $\lambda = \pm \infty$}
\caption{Two failure cases of Aitken's delta-squared process. (left) The first case arises 
if initial $\lambda_0$ is far from the fixed point $\lambda^{\star}$, which results in $\lambda <0$.
 (right) The second case occurs when approximating line (dashed green) is parallel to $y=\lambda$, where $\lambda=\pm\infty$.}
\label{fig:aitken-false}
\end{figure}
\begin{figure}[t]
%\vskip 0.02in
\begin{center}
\centerline{\includegraphics[width=0.9\columnwidth]{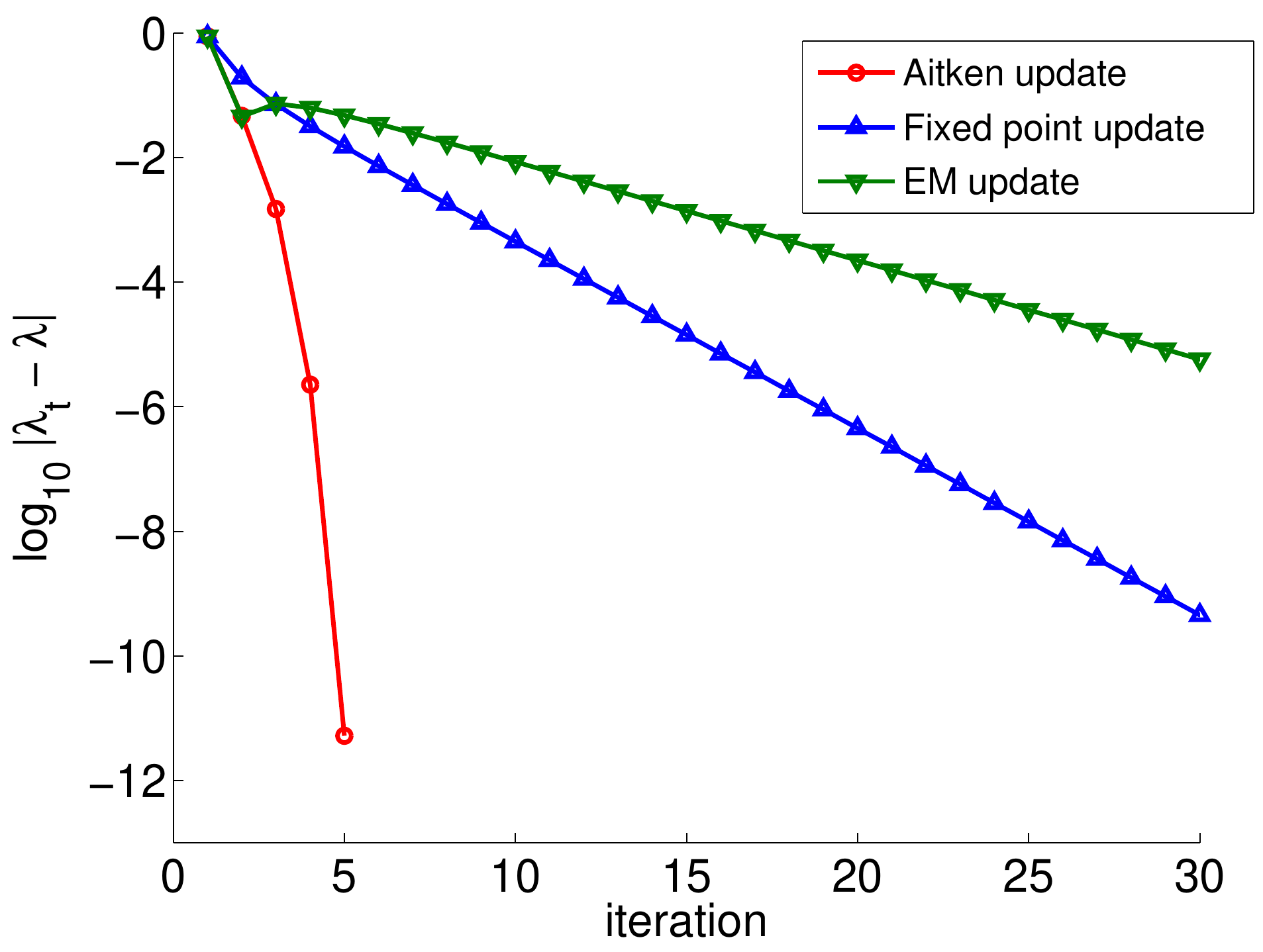} }
%\centerline{\includegraphics[width=0.85\columnwidth]{./fig_speed2.pdf} }
\caption{Comparison between Aitken's delta-squared process, fixed point update rules, 
and EM update rules on PASCAL VOC 2012 dataset (class = \emph{aeroplane}).
Aitken's delta-squared process significantly faster than other methods.}
\label{fig:speed}
\end{center}
\vskip -0.2in
\end{figure}

Figure~\ref{fig:speed} demonstrates the relative convergence rates of three different techniques---Aitken's delta-squared process in Algorithm \ref{alg:FBLS}, fixed point update rules in \eqref{eq:alpha_beta_fixed}, and EM update method, %in \eqref{eq:alpha_em} and \eqref{eq:beta_em}.
where the Aitken's delta-squared process is significantly faster than others for convergence.

%==========================================================
%==========================================================

\section{Experiments}
\label{sec:experiments}
We present the details of our experiment setting and the performance of our algorithm compared to the state-of-the-art techniques in 12 visual recognition benchmark datasets.

\subsection{Datasets and Image Representation}
The benchmark datasets involve various visual recognition tasks such as object recognition, photo annotation, scene recognition, fine grained recognition, visual attribute detection, and action recognition.
%PASCAL VOC 2007/2012, Caltech 101/256, ImageCLEF~\cite{NowakS2011clef}, MIT Indoor~\cite{QuattoniA2009cvpr}, SUN-397~\cite{XiaoJ2010cvpr}, CUB-200~\cite{WahC2011tr}, Oxford Flowers~\cite{NilsbackME2008iccvgip}, UIUC Attributes~\cite{FarhadiA2009cvpr}, Human Attributes~\cite{BourdevLD2011iccv}, and Stanford 40 Actions~\cite{YaoB2011iccv}.
%The list of datasets can be found in Table~\ref{tab:Bayes_vs_SVM}, and we ask readers to refer to the table in the supplementary file to check the details of each dataset.
Table~\ref{tab:data} presents the characteristics of the datasets.
In our experiment, we followed the given train and test split and evaluation measure of each dataset.
For the datasets with bounding box annotations such as CUB200-2011, UIUC object attribute, Human attribute, and Stanford 40 actions, we enlarged the bounding boxes by 150\% to consider neighborhood context as suggested in \cite{RazavianAS2014cvpr,AzizpourH2014arxiv}.

\begin{table*}[t]
\caption{Characteristics of the 12 datasets. $N_1$: number of training data, $N_2$: number of test data, $K$: number of classes, $L$: average number of labels per image, AP: average precision, Acc.: accuracy, AUC: area under the ROC curve.}
%\vspace{-0.3cm}
\label{tab:data}
%\vskip 0.15in
\begin{center}
\begin{small}
%\begin{sc}
\begin{tabular}{l|lrrrrcl}
%\hline
%\abovespace
\belowspace
Dataset & Task & $N_1$ & $N_2$ & $K$ & $L$ & Box  & Measure \\ \hline
\abovespace
PASCAL VOC 2007~\cite{pascal-voc-2007} & object recognition
& 5011 & 4952 & 20 & 1.5 &  & mean AP \\	
PASCAL VOC 2012~\cite{pascal-voc-2012} & object recognition
& 5717 & 5823 & 20 & 1.5 &  & mean AP \\	
Caltech 101~\cite{Fei2007cviu} & object recognition
& 3060 & 6086 & 102 & 1 &  & mean Acc. \\
Caltech 256~\cite{GriffinG2007tr} & object recognition
& 15420 & 15187 & 257 & 1 &  & mean Acc. \\
ImageCLEF 2011~\cite{NowakS2011clef} & photo annotation 
& 8000 & 10000 & 99 & 11.9 &  & mean AP \\
MIT Indoor Scene~ ~\cite{QuattoniA2009cvpr} & scene recognition
& 5360 & 1340 & 67 & 1 &  & mean Acc. \\
SUN 397 Scene~\cite{XiaoJ2010cvpr} & scene recognition
& 19850 & 19850 & 397 & 1 & & mean Acc. \\
CUB 200-2011~\cite{WahC2011tr} & fine-grained recognition 
& 5994 & 5794 & 200 & 1 & $\surd$  & mean Acc.\\
Oxford Flowers~ ~\cite{NilsbackME2008iccvgip} & fine-grained recognition
& 2040 & 6149 & 200 & 1 & & mean Acc.\\
UIUC object attributes~\cite{FarhadiA2009cvpr} & attribute detection 
& 6340 & 8999 & 64 & 7.1 & $\surd$  & mean AUC \\
Human attributes~\cite{BourdevLD2011iccv} & attribute detection
& 4013 & 4022 & 9 & 1.8 & $\surd$ & mean AP \\
\belowspace
Stanford 40 actions~\cite{YaoB2011iccv} & action recognition
& 4000 & 5532 & 40 & 1 & $\surd$ & mean AP \\
\hline
\end{tabular}
%\end{sc}
\end{small}
\end{center}
\end{table*}

%For deep learning image representation, we selected four pre-trained CNNs in the Berkeley's Caffe Model Zoo: Princeton's \emph{GoogLeNet} trained on ImageNet ~\cite{WuZ2014tr}, Princeton's \emph{GoogLeNet} trained on Places ~\cite{WuZ2014tr}, Oxford's \emph{VGG19}-ImageNet trained  on ImageNet ~\cite{SimonyanK2014arxiv}, and Berkeley's \emph{AlexNet} trained on ImageNet ~\cite{DonahueJ2014icml}.
For deep learning representations, we selected 4 pre-trained CNNs from the Caffe Model Zoo: \emph{GoogLeNet}~\cite{WuZ2014tr}, \emph{VGG19}~\cite{SimonyanK2014arxiv}, and \emph{AlexNet}~\cite{DonahueJ2014icml} trained on ImageNet, and \emph{GoogLeNet} trained on Places ~\cite{WuZ2014tr}.
As generic image representations, we used the 4096 dimensional activations of the first fully connected layer in VGG19 and AlexNet and the 1024 dimensional vector obtained from the global average pooling layer located right before the final softmax layer in GoogLeNet.
%which corresponds to the second last layer in \emph{GoogLeNet} and third last layer in \emph{VGG19}-ImageNet and \emph{AlexNet}.

\iffalse
We aggregated image features extracted from multiple locations and scales as in ~\cite{SimonyanK2014arxiv}.
{\color{red}We converted the CNNs to fully-convolutional model in a way similar to ~\cite{SermanetP2014iclr}, then applied the CNN densely over three rescaled images: $256 \times 256$, $384 \times 384$, and $512 \times 512$.
We then performed average pooling on the resulting feature map ($2 \times 2$, $6 \times 6$, and $10 \times 10$) for each scale, and normalized with $L_2$ norm.}
%Our feature extraction process is summarized in Table \ref{tab:deep_feature}.
\fi 

Our implementation is in Matlab2011a, and all experiments were conducted on a quad-core Intel(R) core(TM) i7-3820 @ 3.60GHz processor.

%
%
\iffalse
\subsection{Shared Hyperparameters}

We empirically observed that sharing the regularization parameter for all classes, {\it i.e.}, $\lambda^{(1)} = \cdots = \lambda^{(K)} = \lambda$, is advantageous on multi-class classification problem.
In our Bayesian LS-SVM, sharing regularization parameter can be done easily by sharing the hyperparameters $\alpha^{(1)} = \cdots = \alpha^{(K)} = \alpha$ and $\beta^{(1)} = \cdots = \beta^{(K)} = \beta$.
With these shared hyperparameters, the log evidence with respect to $\lambda$ is modified to
\begin{align}
\calF(\lambda) =~& \frac{K}{2}\sum_{d=1}^D \log \frac{\lambda}{\lambda+s_d} + \frac{NK}{2} \log NK
- \frac{NK}{2} \nonumber \\
& -\frac{NK}{2} \log 2\pi 
 - \frac{NK}{2} \log\left( \sum_{k=1}^K \zeta^{(k)}\right),
\end{align}
where
\bee
\zeta^{(k)} =  \by^{(k)\top}\by^{(k)} -  \sum_{d=1}^D \frac{h_d^{(k)2}}{\lambda + s_d},
\eee
and the corresponding fixed point update rule is given by
\bee
\lambda &=& \frac{ K \sum_{d=1}^D\frac{s_d}{\lambda + s_d}}
{\left( \frac{NK}{ \sum_{k=1}^K \zeta^{(k)}} \right)
\left( \sum_{k=1}^K \sum_{d=1}^D \frac{h_d^{(k)2}}{(\lambda+s_d)^2}\right)}.
\eee 
\fi
%
%

\subsection{Bayesian LS-SVM vs. SVM}
We first compare the performance of our Bayesian LS-SVM with the standard SVM when they are applied to deep CNN features for visual recognition problems.
We used only a single image scale $256 \times 256$ in this experiment.
LIBLINEAR~\cite{FanRE2008jmlr} package is used for SVM training and the regularization parameters are selected by grid search with cross validations.

\iffalse
In this section, we try to answer the following questions:
\begin{itemize}
 \setlength{\itemsep}{1pt}
 \setlength{\parskip}{0pt}
 \setlength{\parsep}{0pt}
\item Is Bayesian LS-SVM competitive with SVM in terms of prediction accuracy and training speed?
\item Can Bayesian LS-SVM guide the selection of deep learning representation for various visual recognition task?
\end{itemize}
To address these questions, 
we extensively evaluated our Bayesian LS-SVM and SVM
with 4 deep learning image representations on 12 visual recognition benchmark datasets.
\fi

Table~\ref{tab:Bayes_vs_SVM} presents the complete results of our experiment.
Bayesian LS-SVM is competitive to SVM in terms of prediction accuracy even with significantly reduced training time.
%{\color{red}Note that training times of the two algorithms are different asymptotically:} $\calO( ND^2 + D^3 + K T D)$ for Bayesian LS-SVM and $\calO(  K T D  N)$ for SVM, where $N$, $D$, $K$, and $T$ denote number of training data, feature dimensions, number of classes, and training iterations respectively.
%Training time of Bayesian LS-SVM increases marginally with respect to the number of classes while that of SVM increases rapidly shown as in Caltech 256 and SUN 397 cases.
Training SVM is getting slower than Bayesian LS-SVM as the number of classes increases so it is particularly slow in Caltech 256 and SUN 397 datasets.
%
% !TEX root = cvpr2016_lssvm.tex

\begin{table*}[t]
\caption{Bayesian LS-SVM versus SVM.
Without time consuming cross validation procedure, Bayesian LS-SVM achieves prediction accuracy competitive to SVM.
In addition, Bayesian LS-SVM selects the proper CNN for each task by using the evidence
(see bold-faced numbers).
Best accuracy in LS-SVM and SVM denotes the maximum achievable accuracy in test dataset using all available learned models. 
Note that the selected model by Bayesian evidence framework or cross validation may not be the best one in testing.
The following sets of regularization parameters are tested for cross validation in LS-SVM and SVM, respectively: $\{2^{-10}, 2^{-9}, \dots, 1, \dots, 2^9, 2^{10} \}$ and $\{ 0.01, 0.05, 0.1, 0.5, 1, 2, 5, 10 \}$.
($\mathsf{G_I}$: {\it GoogLeNet-ImageNet}, $\mathsf{G_p}$: {\it GoogLeNet-Place}, $\mathsf{V}$: {\it VGG19}, and $\mathsf{A}$: {\it AlexNet})}
\label{tab:Bayes_vs_SVM}

%\vskip 0.1in
\begin{center}
\begin{scriptsize}

%%%% to support line break in a cell
\newcommand{\specialcell}[2][c]{%
\begin{tabular}[#1]{@{}c@{}}#2\end{tabular}}

\setlength\tabcolsep{4.5pt} % reduce cell horizontal margin
\begin{tabular}{c|c|crr|cr|c|cr||c|crr|cr|c|cr}
%\hline
%\abovespace
%\belowspace
&\multicolumn{6}{c|}{LS-SVM} & \multicolumn{3}{c||}{SVM} &
\multicolumn{6}{c|}{LS-SVM} & \multicolumn{3}{c}{SVM} \\ \cline{2-19}
\abovespace
& &  \multicolumn{3}{c|}{Bayesian} & \multicolumn{2}{c|}{\specialcell{CV (5 folds)}} & & \multicolumn{2}{c||}{\specialcell{CV (5 folds)}}&
&  \multicolumn{3}{c|}{Bayesian} & \multicolumn{2}{c|}{\specialcell{CV (5 folds)}} & & \multicolumn{2}{c}{\specialcell{CV (5 folds)}} \\ %\abovespace
CNN  & Best & Acc. & Evidence~~ & Time & Acc. & Time  & Best & Acc. & Time & 
Best & Acc. & Evidence~~ & Time & Acc. & Time  & Best & Acc. & Time \\
\hline
%\abovespace
%%%%%%%% VOC07 & SUN 397
&\multicolumn{9}{c||}{PASCAL VOC 2007~\cite{pascal-voc-2007}}& \multicolumn{9}{c}{SUN-397~\cite{XiaoJ2010cvpr}} \\ \cline{2-19} %\abovespace
$\mathsf{G_I}$ & 85.3 & 85.2 & 46.9  $\times 10^3$  & 1.1 & 85.2 & 8.4& 85.0 & 84.7 & 122.4
& 48.1 & 47.0 & 12.8  $\times 10^6$  & 3.1 & 48.1 & 36.5& 54.2 & 54.2 & 8739.6 \\ 
$\mathsf{G_P}$ & 74.1 & 73.8 & 38.6  $\times 10^3$  & 1.0 & 74.0 & 8.1& 74.1 & 73.9 & 144.3
& 61.1 & \textbf{60.1} & \textbf{13.2}  $\mathbf{\times 10^6}$  & 2.9 & 61.1 & 34.4& 63.3 & 63.3 & 8589.4\\ 
$\mathsf{V}$ & 85.9 & \textbf{85.8} & \textbf{48.0}  $\mathbf{\times 10^3}$  & 41.9 & 85.8 & 172.2& 85.9 & 85.8 & 257.5 
&  55.0 & 53.7 & 12.9  $\times 10^6$  & 57.4 & 54.9 & 419.8& 57.1 & 57.1 & 20254.0\\ 
$\mathsf{A}$ & 75.2 & 75.0 & 32.5  $\times 10^3$  & 41.7 & 75.0 & 160.4& 75.3 & 75.2 & 211.1
& 45.4 & 44.9 & 12.7  $\times 10^6$  & 50.8 & 45.4 & 419.0 & 48.6 & 48.6 & 10781.8 \\ \hline
%%%%%%%% VOC12 & CUB200
&\multicolumn{9}{c||}{PASCAL VOC 2012~\cite{pascal-voc-2012}}& \multicolumn{9}{c}{CUB-200~\cite{WahC2011tr}} \\ \cline{2-19} %\abovespace
$\mathsf{G_I}$ & 84.4 & 84.3 & 51.3  $\times 10^3$  & 1.2 & 84.3 & 8.6& 83.9 & 83.7 & 140.8 
& 65.2 & 64.3 & 15.6  $\times 10^5$  & 1.3 & 64.1 & 11.0& 67.6 & 56.5 & 1201.9 \\
$\mathsf{G_P}$ & 73.2 & 72.9 & 40.6  $\times 10^3$  & 1.1 & 73.1 & 8.4& 73.2 & 73.1 & 170.7
& 16.4 & 13.6 & 14.9  $\times 10^5$  & 1.5 & 15.0 & 11.1& 16.8 & 11.1 & 1664.6 \\
$\mathsf{V}$ &  85.2 & \textbf{85.1} & \textbf{52.9}  $\mathbf{\times 10^3}$  & 42.7 & 85.2 & 161.5& 85.6 & 85.4 & 295.9
& 69.2 & \textbf{68.6} & \textbf{15.8} $\mathbf{\times 10^5}$  & 44.1 & 61.5 & 259.2& 71.1 & 59.4 & 2776.2 \\
%\belowspace
$\mathsf{A}$ & 74.1 & 73.9 & 34.3  $\times 10^3$  & 42.7 & 74.0 & 161.8& 74.4 & 74.3 & 160.7
& 59.0 & 58.5 & 15.5  $\times 10^5$  & 45.3 & 46.6 & 257.9& 61.4 & 51.6 & 1645.5\\ \hline 
%%%%%%% Caltech 101 & Flower 102
&\multicolumn{9}{c||}{Caltech 101~\cite{Fei2007cviu}}& \multicolumn{9}{c}{Oxford Flowers~\cite{NilsbackME2008iccvgip}} \\ \cline{2-19} %\abovespace
$\mathsf{G_I}$ & 90.6 & 90.0 & 37.8  $\times 10^4$  & 1.0 & 89.6 & 6.0& 91.4 & 85.1 & 325.0 
& 85.5 & 84.7 & 21.8  $\times 10^4$  & 0.9 & 82.0 & 5.5& 87.4 & 72.0 & 198.8 \\ 
$\mathsf{G_P}$ & 57.0 & 54.3 & 30.6  $\times 10^4$  & 0.9 & 55.1 & 5.9& 57.2 & 41.8 & 390.3
& 55.6 & 51.7 & 19.4  $\times 10^4$  & 0.9 & 51.8 & 5.5& 57.1 & 32.8 & 234.7\\ 
$\mathsf{V}$ & 92.2 & \textbf{92.1} & \textbf{40.9}  $\mathbf{\times 10^4}$  & 31.5 & 88.8 & 142.7& 92.2 & 86.8 & 729.4
& 87.5 & 87.1 & 22.5  $\times 10^4$  & 26.9 & 82.1 & 142.2& 87.6 & 73.4 & 520.9\\ 
%\belowspace
$\mathsf{A}$ & 89.3 & 89.2 & 37.3  $\times 10^4$  & 32.0 & 83.4 & 146.9& 90.0 & 83.5 & 595.3
& 87.6 & \textbf{87.6} & \textbf{22.9}  $\mathbf{\times 10^4}$  & 27.3 & 81.8 & 146.7& 88.3 & 77.1 & 271.3\\ \hline 
%%%%%%% Caltech 256 & UIUC
&\multicolumn{9}{c||}{Caltech 256~\cite{GriffinG2007tr}}& \multicolumn{9}{c}{UIUC Attributes~\cite{FarhadiA2009cvpr}} \\ \cline{2-19} %\abovespace
$\mathsf{G_I}$ & 77.8 & 77.2 & 59.9  $\times 10^5$  & 2.3 & 77.8 & 21.8& 81.2 & 81.2 & 4060.4
& 91.5 & 90.3 & 13.5  $\times 10^4$  & 1.4 & 90.9 & 8.0& 91.3 & 90.6 & 605.5\\
$\mathsf{G_P}$ & 44.9 & 42.6 & 55.9  $\times 10^5$  & 2.2 & 44.9 & 21.2& 48.6 & 48.6 & 4991.8
& 87.8 & 86.6 & 10.5  $\times 10^4$  & 1.3 & 87.1 & 7.4& 88.0 & 87.6 & 726.0 \\
$\mathsf{V}$ & 82.0 & \textbf{81.1} & \textbf{62.3}  $\mathbf{\times 10^5}$  & 52.5 & 81.7 & 339.7& 82.7 & 82.7 & 9653.1
& 92.5 & \textbf{91.1} & \textbf{14.4}  $\mathbf{\times 10^4}$  & 43.8 & 92.0 & 186.3& 92.2 & 91.7 & 1285.4 \\
%\belowspace
$\mathsf{A}$ & 69.7 & 68.9 & 58.6  $\times 10^5$  & 52.9 & 69.7 & 336.9& 72.3 & 72.3 & 5348.6
& 91.4 & 89.9 & 12.9  $\times 10^4$  & 44.1 & 91.0 & 191.2& 90.8 & 90.5 & 683.7 \\ \hline 
%%%%%%% CLEF 256 & Human
&\multicolumn{9}{c||}{ImageCLEF~\cite{NowakS2011clef}}& \multicolumn{9}{c}{Human Attributes~\cite{BourdevLD2011iccv}} \\ \cline{2-19} %\abovespace
$\mathsf{G_I}$ & 49.1 & 48.9 & 20.5  $\times 10^4$  & 1.5 & 48.8 & 37.0& 47.7 & 47.4 & 1218.6
& 76.0 & \textbf{75.8} & \textbf{-74.8}  $\mathbf{\times 10^2}$  & 1.0 & 75.8 & 5.0& 74.2 & 74.1 & 70.6 \\
$\mathsf{G_P}$ & 47.5 & 47.1 & 20.8  $\times 10^4$  & 1.4 & 47.1 & 36.9& 47.1 & 46.7 & 1410.5
& 58.7 & 58.4 & -103.1  $\times 10^2$  & 1.0 & 58.0 & 4.8& 56.9 & 56.5 & 85.5 \\
$\mathsf{V}$ &  50.7 & \textbf{50.3} & \textbf{21.3}  $\mathbf{\times 10^4}$  & 45.9 & 50.4 & 248.5& 50.4 & 50.1 & 2531.2
& 75.4 & 75.1 & -76.0  $\times 10^2$  & 40.3 & 75.2 & 124.2& 73.1 & 72.8 & 131.9 \\
%\belowspace
$\mathsf{A}$ & 44.8 & 44.6 & 18.7  $\times 10^4$  & 46.1 & 44.6 & 245.9& 44.4 & 44.1 & 2140.0 
& 71.9 & 71.3 & -84.4  $\times 10^2$  & 40.7 & 71.7 & 121.2& 70.0 & 69.9 & 63.3\\ \hline 
%%%%%%% MIT67 & Action
&\multicolumn{9}{c||}{MIT Indoor~\cite{QuattoniA2009cvpr}}& \multicolumn{9}{c}{Stanford 40 Action~\cite{YaoB2011iccv}} \\ \cline{2-19} %\abovespace
$\mathsf{G_I}$ & 66.7 & 66.0 & 30.1  $\times 10^4$  & 1.2 & 66.7 & 5.8& 69.4 & 69.2 & 400.9
& 70.2 & 69.8 & 100.4  $\times 10^3$  & 1.0 & 69.6 & 11.6& 69.8 & 69.6 & 211.7\\
$\mathsf{G_P}$ & 80.0 & \textbf{79.9} & \textbf{35.2}  $\mathbf{\times 10^4}$  & 1.1 & 80.0 & 5.8& 81.1 & 80.4 & 402.5 
& 48.3 & 47.6 & 86.5  $\times 10^3$  & 1.1 & 47.9 & 11.4& 48.2 & 47.7 & 246.2 \\
$\mathsf{V}$ & 73.2 & 73.1 & 31.1  $\times 10^4$  & 42.6 & 73.2 & 186.8& 74.7 & 74.7 & 895.5
& 75.4 & \textbf{75.2} & \textbf{109.3}  $\mathbf{\times 10^3}$  & 41.1 & 75.1 & 142.9& 75.8 & 75.3 & 418.7\\
%\belowspace
$\mathsf{A}$ & 62.0 & 61.1 & 28.6  $\times 10^4$  & 42.2 & 60.5 & 187.4& 63.1 & 63.1 & 460.9
& 58.0 & 57.7 & 89.6  $\times 10^3$  & 41.5 & 57.5 & 156.5& 57.4 & 57.1 & 206.8 \\ \hline
\end{tabular}
\end{scriptsize}
\end{center}
\vskip -0.1in
\end{table*}

Another notable observation in Table~\ref{tab:Bayes_vs_SVM} is that the order of prediction accuracy is highly correlated to the evidence.
%This means we can select a proper deep learning image representation for each dataset without time consuming cross validation.
This means that the selected model by Bayesian LS-SVM produces reliable testing accuracy and a proper deep learning image representation is obtained without time consuming grid search and cross validation.
Note that cross validations in LS-SVM and SVM play the same role, but are less reliable and slower than our Bayesian evidence framework.
The capability to select the appropriate CNN model and the corresponding regularization parameter is one of the most important properties of our algorithm.

\subsection{Comparison with Other Methods}
%Because several pre-trained CNNs are available in public, one natural question is how to aggregate multiple CNNs to achieve the state-of-the-art performance on target recognition task.
We now show that our Bayesian LS-SVM identifies a combination of multiple CNNs to improve accuracy without grid search and cross validation.
For each task, we select a subset of 4 pre-trained CNNs in a greedy manner; we add CNNs to our selection, one by one, until the evidence does not increase.
%For each selected CNN, we aggregate image features from multiple locations and scales as in ~\cite{SimonyanK2014arxiv}, where three rescaled images ($256 \times 256$, $384 \times 384$, and $512 \times 512$) are used and average pooling is performed for each scale.
Our algorithm is compared with DeCAF~\cite{DonahueJ2014icml}, Zeiler~\cite{ZeilerMD2014eccv}, INRIA~\cite{OquabM2014cvpr}, KTH-S~\cite{RazavianAS2014cvpr}, KTH-FT~\cite{AzizpourH2014arxiv}, VGG~\cite{SimonyanK2014arxiv}, Zhang~\cite{ZhangN2014cvpr,ZhangN2014eccv}, and TUBFI~\cite{binder2011joint}.
In addition, our ensembles identified by greedy evidence maximization are compared with the oracle combinations---the ones with the highest accuracy in test set found by exhaustive search---and the best combinations found by exhaustive evidence maximization.

\iffalse
Table~\ref{tab:state-of-the-art} presents that our ensembles approach achieves the best performance in most of the 12 tasks.
Note that our network selection is natural and reasonable; GoogLeNet-ImageNet and VGG19 are selected frequently while GoogLeNet-Place is preferred to GoogLeNet-ImageNet in MIT Indoor and SUN-397 since the datasets are constructed for scene recognition.
Sometimes our greedy approach fail to utilize complementary information from ImageNet and Place together.
For example, in the case of Stanford 40 actions, 
although the accuracy of GoogLeNet-Place is low,
it has complementary information 
(e.g. a fair portion of \emph{brushing teeth} images were taken from bathrooms),
hence it should have been selected.
However in most cases, the identified ensembles by greedy approach are consistent with the selections by exhaustive evidence maximization and even oracle selections\footnote{This option is practically impossible since it requires evaluation with test dataset using all available models for model selection.} made by testing accuracy maximization.
\fi

Table~\ref{tab:state-of-the-art} presents that our ensembles approach achieves the best performance in most of the 12 tasks.
The identified ensembles by the greedy approach are consistent with the selections by exhaustive evidence maximization and even oracle selections\footnote{This option is practically impossible since it requires evaluation with test dataset using all available models for model selection.} made by testing accuracy maximization.
Note that our network selections are natural and reasonable; GoogLeNet-ImageNet and VGG19 are selected frequently while GoogLeNet-Place is preferred to GoogLeNet-ImageNet in MIT Indoor and SUN-397 since the datasets are constructed for scene recognition.
It turns out that the proposed algorithm tends to choose the networks with higher accuracies in the target task even though it makes selections based only on the evidence in a greedy manner.
An interesting observation is that our result is less consistent with the selections by oracle and exhaustive evidence maximization in Stanford 40 Actions dataset, where GoogLeNet-Place seems to provide complementary information even with its low accuracy and is helpful to improve recognition performance.
It is probably because actions are frequently performed at typical places, \eg, a fair portion of images in \emph{brushing teeth} class are taken from bathrooms.

%The selected network combinations are specified in the last row of Table~\ref{tab:state-of-the-art}, where C-S1 and C-S3 denote our Bayesian LS-SVM results with a single ($256 \times 256$) and multiple scales ($256 \times 256$, $384 \times 384$, $512 \times 512$), respectively.
%Our method generally performs better with multiple scales.

%{\color{red}This shows that the greedy method is a reasonable choice as it takes only $O(n)$, resenting performance comparable to the best one found by exhaustive methods, where $n$ is the number of CNNs. While we can apply the same strategy for traditional SVMs, it takes $O(mn)$, where $m$ is the number of hyperparameter settings.} \\
  
%
\begin{table*}[ht] 
\vskip -0.1in
\caption{Comparison to existing methods in the 12 benchmark datasets. 
The best ensembles identified by maximizing evidence through exhaustive search mostly coincide with the oracle combinations---the ones with the highest accuracy in test set, which is also found by exhaustive search.
The ensembles identified by our greedy search are very similar to the ones by these exhaustive search methods, and our algorithm consequently performs best in many tested datasets.
We used three scales $\{256, 384,512\}$ as done in \cite{SimonyanK2014arxiv},
where we simply averaged the prediction scores from three scales. 
%The performance is further improved if multiple scales are considered, where we simply averaged the prediction scores from 3 scales. 
%($\mathsf{G_I}$: {\it GoogLeNet-ImageNet}, $\mathsf{G_p}$: {\it GoogLeNet-Place}, $\mathsf{V}$: {\it VGG19}, and $\mathsf{A}$: {\it AlexNet})
}
%\vspace{-0.2cm}
\label{tab:state-of-the-art}
%\vskip -0.15in
\begin{center} 
%\begin{scriptsize}
\begin{footnotesize}
\begin{tabular}{l|cccccccccccc}
%\belowspace
Method & VOC07 & VOC12 & CAL101 & CAL256 & CLEF & MIT & SUN & Birds & Flowers & UIUC & Human & Action \\
\hline
\abovespace
DeCAF  & - & - & 86.9 & - & - & - & 38.0 & 65.0 & - & - & - & - \\
Zeiler & - & 79.0 & 86.5 & 74.2 & - & - & - & -& -& -& -& - \\
INRIA & 77.7 & 82.8  &-& - & - &- & - &-&-&-&-&-\\
KTH-S & 71.8 & - & -&- & - & 64.9 & 49.6 & 62.8 & 90.5 & 90.6 & 73.8 & 58.9 \\
KTH-FT & 80.7 & &-&-- & - & 71.3 & 56.0 & 67.1 & 91.3 & 91.5 & 74.6 & 66.4\\
VGG & 89.7 & 89.3 &  92.7 & \textbf{86.2} & - & - & - &- & - & -& -& - \\
Zhang &-&-& - & - & - & - & - & 76.4 & - & - & 79.0 & - \\ 
TUBFI &-&-& - & - & 44.3 & - & - & - & - & - & - & - \\ \hline
\abovespace
$\mathsf{G_I}$ & 87.5 & 86.2 & 90.5 & 77.7 & 50.3 & 71.3 & 48.3 & 64.7 & 88.1 & 91.1 & 78.4 & 71.0\\
$\mathsf{G_P}$ & 75.7 & 74.9 & 53.8 & 42.1 & 48.1 & 80.8 & 59.8 & 14.9 & 57.8 & 87.3 & 59.7 & 48.4\\
$\mathsf{V}$ & 88.4 & 87.8 & 93.3 & 83.3 & 52.4 & 77.8 & 56.1 & 69.9 & 91.5 & 91.8 & 79.1 & 77.0 \\
$\mathsf{A}$ & 75.0 & 73.9 & 88.3 & 69.7 & 52.3 & 77.5 & 42.4 & 60.7 & 86.7 & 89.9 & 71.3 & 57.7\\ \hline
\abovespace
Oracle & $\mathsf{G_I G_P V}$ & $\mathsf{G_I G_P V}$ & $\mathsf{G_I V A}$ & $\mathsf{G_I G_P V A}$ & $\mathsf{G_I G_P V A}$ & $\mathsf{G_P V A}$ & $\mathsf{G_I G_P V A}$ & $\mathsf{G_I V A}$ & $\mathsf{G_I G_P V A}$ & $\mathsf{G_I V A}$ & $\mathsf{G_I V A}$ & $\mathsf{G_I G_P V}$ \\
% & 88.4 & 87.9 & 95.4 & 84.5 & 54.5 & 83.9 & 66.7 & 73.8 & 93.0 & 91.7 & 79.0 & 77.9\\ 
(exhaustive) & 90.0 & 89.4 & 95.3 & 86.1 & 55.7 & 84.9 & 67.5 & 77.3 & 94.7 & 92.0 & 80.8 & 78.6 \\ \hline 
\abovespace
Max evid. & $\mathsf{G_I G_P V}$ & $\mathsf{G_I G_P V}$ & $\mathsf{G_I V A}$ & $\mathsf{G_I G_P V A}$ & $\mathsf{G_I G_P V}$ & $\mathsf{G_P V}$ & $\mathsf{G_P V}$ & $\mathsf{G_I V A}$ & $\mathsf{G_I V A}$ & $\mathsf{G_I G_P V A}$ & $\mathsf{G_I V A}$ & $\mathsf{G_I G_P V}$ \\
% & 88.4 & 87.9 & 95.4 & 84.5 & 54.4 & 82.6 & 65.9 & 73.8 & 92.6 & 91.7 & 79.0 & 77.9 \\ 
(exhaustive) & 90.0 & 89.4 & 95.3 & 86.1 & 55.5 & 84.7 & 67.5 & 77.3 & 94.5 & 92.0 & 80.8 & 78.6 \\ \hline 
\abovespace
Ours & $\mathsf{G_I G_P V}$ & $\mathsf{G_I G_P V}$ & $\mathsf{G_I V A}$ & $\mathsf{G_I G_P V A}$ & $\mathsf{G_I G_P V}$ & $\mathsf{G_P V}$ & $\mathsf{G_P V}$ & $\mathsf{G_I V A}$ & $\mathsf{G_I V A}$ & $\mathsf{G_I G_P V A}$ & $\mathsf{G_I V A}$ & $\mathsf{G_I V A}$\\ 
% & 88.4 & 87.9 & 95.4 & 84.5 & 54.4 &  82.6 &  65.9 & 73.8 & 92.6 & 91.7 & 79.0 & 77.6 \\
(greedy)  & \textbf{90.0} & \textbf{89.4} & \textbf{95.3} & 86.1 & \textbf{55.5} & \textbf{84.7} & \textbf{67.5} & \textbf{77.3} & \textbf{94.5} & \textbf{92.0} & \textbf{80.8} & \textbf{77.8} \\   \hline
\end{tabular}
%\end{scriptsize}
\end{footnotesize}
\end{center}
\vskip -0.1in
\end{table*}

%Specifically, we add a network with the largest evidence in each stage and test if the augmented network improves the evidence.
%The network is accepted if the evidence increases, or rejected otherwise.
%After the last candidate is tested, we obtain the final combination, concatenate the feature vectors extracted from all the accepted networks, and learn a new model based on the concatenated features using our Bayesian LS-SVM.

\iffalse
Note that we have 12 deep learning image representation (4 CNNs $\times$ 3 scales). 
In the same level of scales, we concatenated features from different CNN if the evidence is increased when Bayesian LS-SVM is trained with this new features.
The features from \emph{GoogLeNet} pre-trained on ImageNet, \emph{VGG19} pre-trained on ImageNet were always selected  and \emph{AlexNet} pre-trained on ImageNet were never selected.
The features from \emph{GoogLeNet} pre-trained on Place were selected excepted for 2 datasets, which are CUB200-2011 and Human attribute.
Our C-S1 denotes Bayesian LS-SVM results of concatenated features extracted from image scale $256 \times 256$.
The performance can be further improved by simply averaging prediction scores from multiple scales
(Our C-S3).
\fi

%\FloatBarrier 

\section{Conclusion}
We described a simple and efficient technique to transfer deep CNN models pre-trained on specific image classification tasks to another tasks.
Our approach is based on Bayesian LS-SVM, which combines Bayesian evidence framework and SVM with a least squares loss.
In addition, we presented a faster fixed point update rule for evidence maximization through Aitken's delta-squared process.
%We presented a new fixed point update rule for maximizing evidence and the existing condition of fixed point, which is easily satisfied with common $L_2$ normalization.
%We also presented a speed up algorithm based on Aitken's delta-squared process.
Our fast Bayesian LS-SVM demonstrated competitive results compared to the standard SVM by selecting a deep CNN model in 12 popular visual recognition problems.
%Our fast Bayesian LS-SVM obtained competitive results compared to the standard SVMs in the 12 benchmark datasets by selecting a deep CNN model effectively and estimating corresponding regularization parameters automatically.
We also achieved the state-of-the-art performance by identifying a good ensemble of the candidate models through our Bayesian LS-SVM framework.

\subsection*{Acknowledgements}
\iffalse
This work was supported by the IT R\&D Program of MSIP/IITP (B0101-15-0307, Machine Learning Center), (B0101-15-0552, DeepView) and National Research Foundation (NRF) of Korea (NRF-2013R1A2A2A01067464).
\fi

This work was partly supported by Institute for Information \& Communications Technology Promotion (IITP) grant funded by the Korea government (MSIP) [B0101-16-0307; Basic Software Research in Human-level Lifelong Machine Learning 
(Machine Learning Center), B0101-16-0552; Development of Predictive Visual Intelligence Technology (DeepView)], and National Research Foundation (NRF) of Korea [NRF-2013R1A2A2A01067464].

{\small
\bibliographystyle{ieee}
\bibliography{sjc}

\begin{thebibliography}{10}\itemsep=-1pt

\bibitem{AitkenAC1927PRSE}
A.~C. Aitken.
\newblock On {Bernoulli's} numerical solution of algebraic equations.
\newblock {\em Proceedings of the Royal Society of Edinburgh}, 46:289--305,
  1927.

\bibitem{AzizpourH2014arxiv}
H.~Azizpour, A.~S. Razavian, J.~Sulivan, A.~Maki, and S.~Carlsson.
\newblock From generic to specific deep representations for visual recognition.
\newblock In {\em CVPR Workshops}, 2015.

\bibitem{binder2011joint}
A.~Binder, W.~Samek, M.~Kloft, C.~M{\"u}ller, K.-R. M{\"u}ller, and
  M.~Kawanabe.
\newblock The joint submission of the {TU Berlin and Fraunhofer FIRST}
  ({TUBFI}) to the {ImageCLEF}2011 photo annotation task.
\newblock 2011.

\bibitem{BishopCM95book}
C.~M. Bishop.
\newblock {\em Neural Networks for Pattern Recognition}.
\newblock Clarendon press Oxford, 1995.

\bibitem{BourdevLD2011iccv}
L.~D. Bourdev, S.~Maji, and J.~Malik.
\newblock Describing people: A poselet-based approach to attribute
  classification.
\newblock In {\em ICCV}, 2011.

\bibitem{Chatfield2014bmvc}
K.~Chatfield, K.~Simonyan, A.~Vedaldi, and a.~Zisserman.
\newblock Return of the devil in the details: Delving deep into convolutional
  nets.
\newblock In {\em BMVC}, 2014.

\bibitem{DonahueJ2014icml}
J.~Donahue, Y.~Jia, O.~Vinyals, J.~Hoffman, n.~Zhang, E.~Tzeng, and T.~Darrell.
\newblock {DeCAF}: A deep convolutional activation feature for generic visual
  recognition.
\newblock In {\em ICML}, 2014.

\bibitem{pascal-voc-2007}
M.~Everingham, L.~V. Gool, C.~K.~I. Williams, J.~Winn, and A.~Zisserman.
\newblock The {PASCAL} {V}isual {O}bject {C}lasses {C}hallenge 2007 {(VOC
  2007)} {R}esults, 2007.

\bibitem{pascal-voc-2012}
M.~Everingham, L.~V. Gool, C.~K.~I. Williams, J.~Winn, and A.~Zisserman.
\newblock The {PASCAL} {V}isual {O}bject {C}lasses {C}hallenge 2012 {(VOC
  2012)} {R}esults, 2012.

\bibitem{FanRE2008jmlr}
R.~E. Fan, K.~W. Chang, C.~J. Hsieh, X.~R. Wang, and C.~J. Lin.
\newblock {LIBLINEAR}: A library for large linear classification.
\newblock {\em JMLR}, 9:1871--1874, 2008.

\bibitem{FarhadiA2009cvpr}
A.~Farhadi, I.~Endres, D.~Hoiem, and D.~A. Forsyth.
\newblock Describing objects by their attributes.
\newblock In {\em CVPR}, 2009.

\bibitem{Fei2007cviu}
L.~Fei-Fei, R.~Fergus, and P.~Perona.
\newblock Learning generative visual models from few training examples: An
  incremental bayesian approach tested on 101 object categories.
\newblock {\em CVIU}, 106(1):59--70, 2007.

\bibitem{girshick2014cvpr}
R.~Girshick, J.~Donahue, T.~Darrell, and J.~Malik.
\newblock Rich feature hierarchies for accurate object detection and semantic
  segmentation.
\newblock In {\em CVPR}, pages 580--587. IEEE, 2014.

\bibitem{GriffinG2007tr}
G.~Griffin, A.~Holub, and P.~Perona.
\newblock Caltech-256 object category dataset.
\newblock Technical report, California Institute of Technology, 2007.

\bibitem{Nam16}
B.~H. H.~Nam.
\newblock Learning multi-domain convolutional neural networks for visual
  tracking.
\newblock In {\em CVPR}, 2016.

\bibitem{noh2015learning}
B.~H. H.~Noh, S.~Hong.
\newblock Learning deconvolution net- work for semantic segmentation.
\newblock In {\em ICCV}, 2015.

\bibitem{KrizhevskyA2012nips}
A.~Krizhevsky, I.~Sutskever, and G.~E. Hinton.
\newblock {ImageNet} classification wit deep convolutional neural networks.
\newblock In {\em NIPS}, volume~25, 2012.

\bibitem{MacKayDJC92nc}
D.~J.~C. MacKay.
\newblock Bayesian interpolation.
\newblock {\em Neural Computation}, 4(3):415--447, 1992.

\bibitem{NilsbackME2008iccvgip}
M.-E. Nilsback and A.~Zisserman.
\newblock Automated flower classification over a large number of classes.
\newblock In {\em Proceedings of the Indian Conference on Computer Vision,
  Graphics and Image Processing}, 2008.

\bibitem{NowakS2011clef}
S.~Nowak, K.~Nagel, and J.~Liebetrau.
\newblock The {CLEF} 2011 photo annotation and concept-based retrieval tasks.
\newblock In {\em CLEF Workshop Notebook Paper}, 2011.

\bibitem{OquabM2014cvpr}
M.~Oquab, L.~Bottou, I.~Laptev, and J.~Sivic.
\newblock Learning and transferring mid-level image representations using
  convolutional neural networks.
\newblock In {\em CVPR}, 2014.

\bibitem{QuattoniA2009cvpr}
A.~Quattoni and A.~Torrabla.
\newblock Recognizing indoor scenes.
\newblock In {\em CVPR}, 2009.

\bibitem{RazavianAS2014cvpr}
A.~S. Razavian, H.~Azizpour, J.~Sullivan, and S.~Carlsson.
\newblock {CNN} features off-the-shelf: An astounding baseline for recognition.
\newblock In {\em CVPR Workshops}, 2014.

\bibitem{SermanetP2014iclr}
P.~Sermanet, D.~Eigen, X.~Zhang, M.~Mathieu, R.~Fergus, and Y.~LeCun.
\newblock {OverFeat}: Integrated recognition, localization and detection using
  convolutional networks.
\newblock In {\em ICLR}, 2014.

\bibitem{SimonyanK2014arxiv}
K.~Simonyan and A.~Zisserman.
\newblock Very deep convolutional networks for large-scale image recognition.
\newblock In {\em ICLR}, 2015.

\bibitem{SuykensJAK99npl}
J.~A.~K. Suykens and J.~Vandewalle.
\newblock Least squares support vector machine classifiers.
\newblock {\em Neural Processing Letters}, 9(3):293--300, 1999.

\bibitem{SzegedyC2014arxiv}
C.~Szegedy, W.~Liu, Y.~Jia, P.~Sermanet, D.~A. S.~Reed, D.~Erhan, V.~Vanhoucke,
  and A.~Rabinovich.
\newblock Going deeper with convolutions.
\newblock In {\em CVPR}, 2015.

\bibitem{GestelTV2004mlj}
T.~{Van Gestel}, J.~A. K. S.~B. Baesems, S.~Viaene, J.~Vanthienen, G.~Dedene,
  B.~{De Moor}, and J.~Vandewalle.
\newblock Benchmarking least squares support vector machines classifiers.
\newblock {\em Machine Learning}, 54(1):5--32, 2004.

\bibitem{GestelTV2002nc}
T.~{Van Gestel}, J.~A.~K. Suykens, G.~Lanckrie, A.~Lambrechts, B.~D. Moor, and
  J.~Vandewalle.
\newblock Bayesian framework for least-squares support vector machine
  classifiers, gaussian processes, and kernel fisher discriminant analysis.
\newblock {\em Neural Computation}, 14(5):1115--1147, 2002.

\bibitem{WahC2011tr}
C.~Wah, S.~Branson, P.~Welinder, P.~Perona, and S.~Belongie.
\newblock The {Caltech-UCSD Birds-200-2011} dataset.
\newblock Technical report, California Institute of Technology, 2011.

\bibitem{WuZ2014tr}
Z.~Wu, Y.~Zhang, F.~Yu, and J.~Xiao.
\newblock A {GPU} implementation of {GoogLeNet}.
\newblock Technical report, Princeton University, 2014.

\bibitem{XiaoJ2010cvpr}
J.~Xiao, J.~Hays, K.~A. Ehinger, A.~Oliva, and A.~Torrabla.
\newblock {SUN} database: Large-scale scene recognition from abbey to zoo.
\newblock In {\em CVPR}, 2010.

\bibitem{YaoB2011iccv}
B.~Yao, X.~Jiang, A.~Khosla, A.~L. Lin, L.~J. Guibas, and L.~Fei-Fei.
\newblock Action recognition by learning bases of action attributes and parts.
\newblock In {\em ICCV}, 2011.

\bibitem{ZeilerMD2014eccv}
M.~D. Zeiler and R.~Fergus.
\newblock Visualizing and understanding convolutional networks.
\newblock In {\em ECCV}, 2014.

\bibitem{ZhangN2014cvpr}
N.~Zhang, , M.~Paluri, M.~Ranzato, T.~Darrell, and L.~Bourdev.
\newblock {PANDA}: Pose aligned networks for deep attribute modeling.
\newblock In {\em CVPR}, 2014.

\bibitem{ZhangN2014eccv}
N.~Zhang, J.~Donahue, R.~Girshick, and T.~Darrell.
\newblock Part-based {R-CNNs} for fine-grained category detection.
\newblock In {\em ECCV}, 2014.

\bibitem{ZhangP2004icpr}
P.~Zhang and J.~Peng.
\newblock {SVM} vs regularized least squares classification.
\newblock In {\em ICPR}, 2004.

\end{thebibliography}
}

\end{document}